\documentclass[twoside,11pt]{article}
\input epsf 

\usepackage{jair}
\usepackage{theapa}
\usepackage{graphics}

\jairheading{30}{2007}{101-132}{10/05}{9/07}
\ShortHeadings{The Planning Spectrum --- One, Two, Three, Infinity}
{Pistore \& Vardi}
\firstpageno{101}

\title{The Planning Spectrum --- One, Two, Three, Infinity}

\author{\name Marco Pistore \email pistore@dit.unitn.it\\
        \addr Department of Information and Communication Technology\\
        University of Trento\\
        Via Sommarive 14, 38050 Povo (Trento), Italy
        \AND
        \name Moshe Y. Vardi \email vardi@cs.rice.edu\\
        \addr Department of Computer Science\\
        Rice University\\
        6100 S. Main Street, Houston, Texas
        }

\usepackage{amsmath}
\usepackage{amssymb}
\usepackage{xspace}
\usepackage[matrix,arrow,curve,tips,frame]{xy}

\newtheorem{definition}{Definition}
\newtheorem{proposition}{Proposition}
\newtheorem{theorem}[proposition]{Theorem}
\newtheorem{lemma}[proposition]{Lemma}
\newtheorem{example}{Example}

\newenvironment{proof}%
  {\trivlist\item[]\textbf{Proof.\ }\parindent=0pt}%
  {\endtrivlist}
\newcommand{\qed}{\nobreak\hfill\nobreak$\Box$}
  
\newcommand{\partialto}{\rightharpoonup}
\newcommand{\dom}{\mathop{\mathrm{dom}}}
\newcommand{\Nat}{\mathbb{N}}
\newcommand{\aequiv}{\sim}
\newcommand{\aimply}{\rightsquigarrow}
\newcommand{\as}{\mathrel{::=}}
\newcommand{\orbar}{\mathrel{|}}
\newcommand{\arity}{\textit{arity}}
\newcommand{\st}{\mathrel{:}}
\newcommand{\Tau}{\text{\Large $\tau$}}
\newcommand{\Tauw}{\Tau{\!}_w}
\newcommand{\Dcal}{\mathcal{D}}

\newcommand{\pf}[1]{[[#1]]}

\newcommand{\LTL}{\text{LTL}\xspace}
\newcommand{\CTLstar}{\text{CTL\!*}\xspace}

\newcommand{\X}{\mathop{\textrm{\normalfont{X}}}}
\newcommand{\U}{\mathop{\textrm{\normalfont{U}}}}
\newcommand{\F}{\mathop{\textrm{\normalfont{F}}}}
\newcommand{\G}{\mathop{\textrm{\normalfont{G}}}}

\newcommand{\Ap}{\mathop{\textrm{\normalfont{A}}}}
\newcommand{\Ep}{\mathop{\textrm{\normalfont{E}}}}

\newcommand{\AG}{\mathop{\textrm{\normalfont{AG}}}}

\newcommand{\EF}{\mathop{\textrm{\normalfont{EF}}}}

\newcommand{\EXG}{\mathop{\textrm{\normalfont{EXG}}}}

\renewcommand{\AA}{{\!\mathcal{A}}}
\newcommand{\EE}{{\mathcal{E}}}
\renewcommand{\AE}{\mathcal{AE}}

\newcommand{\init}{\mathit{init}}

\newcommand{\Out}{\mathcal{P}}

\begin{document}


\maketitle


\begin{abstract}
  Linear Temporal Logic (LTL) is widely used for defining conditions on
  the execution paths of dynamic systems.
  In the case of dynamic systems that allow for nondeterministic
  evolutions, one has to specify, along with an LTL formula $\varphi$,
  which are the paths that are required to satisfy the formula.
  Two extreme cases are the \emph{universal} interpretation $\AA .
  \varphi$, which requires that the formula be satisfied for all 
  execution paths, and the \emph{existential} interpretation $\EE .
  \varphi$, which requires that the formula be satisfied for some execution
  path.
  
  When LTL is applied to the definition of goals in planning problems
  on nondeterministic domains, these two extreme cases are too
  restrictive.
  It is often impossible to develop plans that achieve the goal in all
  the nondeterministic evolutions of a system, and it is too weak to
  require that the goal is satisfied by some execution.
  
  In this paper we explore alternative interpretations of an LTL
  formula that are between these extreme cases.
  We define a new language that permits an arbitrary combination of
  the $\AA$ and $\EE$ quantifiers, thus allowing, for instance, to
  require that each finite execution can be extended to an execution
  satisfying an LTL formula ($\AA\EE . \varphi$), or that there is
  some finite execution whose extensions all satisfy an LTL formula
  ($\EE\AA . \varphi$).
  We show that only eight of these
  combinations of path quantifiers are relevant,
  corresponding to an alternation of the quantifiers of length one
  ($\AA$ and $\EE$), two ($\AA\EE$ and $\EE\AA$), three ($\AA\EE\AA$
  and $\EE\AA\EE$), and infinity ($(\AA\EE)^\omega$ and
  $(\EE\AA)^\omega$).
  We also present a planning algorithm for the new language
  that is based on an automata-theoretic approach, and study its
  complexity.
\end{abstract}


\section{Introduction}
\label{sec:intro}

In automated task planning \cite{FN71,PW92,GNT04}, given a description of a
dynamic domain and of the basic actions that can be performed on it,
and given a goal that defines a success condition to be achieved, one
has to find a suitable plan, that is, a description of the actions to
be executed on the domain in order to achieve the goal.
``Classical'' planning concentrates on the so called ``reachability''
goals, that is, on goals that define a set of final desired states to
be reached.
Quite often practical applications require plans that deal with goals
that are more general than sets of final states. Several planning
approaches have been recently proposed, where \emph{temporal logic}
formulas are used as goal language, thus allowing for goals that
define conditions on the whole plan execution paths, i.e., on the
sequences of states resulting from the execution of plans
\cite{BK98,BK00,CdGLV02,CM98,DLPT02,dGV99,KD01,PT01}.
Most of these approaches use \emph{Linear Temporal Logic} 
(LTL) \cite{Eme90} as the goal language. LTL allows one to express
reachability goals (e.g., $\F q$ --- reach $q$), maintainability goals
(e.g., $\G q$ --- maintain $q$), as well as goals that combine
reachability and maintainability requirements (e.g., $\F \G q$ --- reach a
set of states where $q$ can be maintained), and Boolean combinations
of these goals.

In \emph{planning in nondeterministic domains}
\cite{CPRT03,Peo92,War76}, actions are allowed to have different
outcomes, and it is not possible to know at planning time which of the
different possible outcomes will actually take place.
Nondeterminism in action outcome is necessary for modeling in a
realistic way several practical domains, ranging from robotics to
autonomous controllers to two-player games.\footnote{See the work of \citeA{GNT04} 
for a deeper discussion on the fundamental role of nondeterminism in
planning problems and in practical applications.}
For instance, in a realistic robotic application one has to take into
account that actions like ``pick up object'' might result in a failure
(e.g., if the object slips out of the robot's hand).
A consequence of nondeterminism is that the execution of a plan may
lead to more than one possible execution path.
Therefore, one has to distinguish whether a given goal has to be
satisfied by all the possible execution paths (in this case we speak
of ``strong'' planning), or only by some of the possible execution
paths (``weak'' planning).
In the case of an LTL goal $\varphi$, strong planning corresponds to
interpreting the formula in a universal way, as $\AA . \varphi$, while weak
planning corresponds to interpreting it in an existential way, as $\EE .
\varphi$.

Weak and strong plans are two extreme ways of satisfying
an LTL formula.
In nondeterministic planning domains, it might be impossible to
achieve goals in 
a strong way: for instance, in the robotic application it might be
impossible to fulfill a given task if objects keep slipping from the
robot's hand.
On the other hand, weak plans are too unreliable, since they achieve
the goal only under overly optimistic assumptions on the outcomes of action
executions.

In the case of reachability goals, \emph{strong cyclic planning}
\cite{CPRT03,DTV99} has been shown to provide a viable compromise
between weak and strong planning.
Formally, a plan is strong cyclic if each possible partial execution
of the plan can always be extended to an execution that reaches some
goal state.
Strong cyclic planning allows for plans that encode iterative
trial-and-error strategies, like ``pick up an object until succeed''.
The execution of such strategies may loop forever only in the case the
action ``pick up object'' continuously fails, and a failure in
achieving the goal for such an unfair execution is usually
acceptable.
Branching-time logics like CTL and \CTLstar allow for expressing goals
that take into account nondeterminism. Indeed, \citeA{DTV99} show how
to encode strong cyclic reachability goals as CTL formulas.
However, in CTL and \CTLstar path quantifiers are interleaved with
temporal operators, making it difficult to extend
the encoding of strong cyclic planning proposed by \citeA{DTV99}
to generic temporal goals.

In this paper we define a new logic that allows for exploring the
different degrees in which an LTL formula $\varphi$ can be satisfied
that exist between the strong goal $\AA . \varphi$ and the weak goal
$\EE . \varphi$.
We consider logic formulas of the form $\alpha .
\varphi$, where $\varphi$ is an LTL formula and $\alpha$ is a path
quantifier that generalizes the $\AA$ and $\EE$ quantifiers used for
strong and weak planning.
A path quantifier is a (finite or infinite) word on alphabet $\{
{\AA}, {\EE} \}$.
The path quantifier can be seen as the definition of a two-player
game for the selection of the outcome of action execution.
Player A (corresponding to symbol $\AA$) chooses the action outcomes
in order to make goal $\varphi$ fail, while player E (corresponding to
symbol $\EE$) chooses the action outcomes in order to satisfy the goal
$\varphi$.
At each turn, the active player controls the outcome of action execution
for a finite number of actions and then passes the control to the
other player.\footnote{If the path quantifier is a finite
word, the player that has the last turn chooses the action outcome for
the rest of the infinite execution.}
We say that a plan satisfies the goal $\alpha . \varphi$ if the player
E has a winning strategy, namely if, for all the possible moves of
player A, player E is always able to build an execution path that
satisfies the LTL formula $\varphi$.

Different path quantifiers define different alternations in the turns
of players A and E.
For instance, with goal $\AA . \varphi$ we require that the formula
$\varphi$ is satisfied independently of how the ``hostile'' player A
chooses the outcomes of actions, that is, we ask for a strong plan.
With goal $\EE . \varphi$ we require that the formula $\varphi$ is
satisfied for some action outcomes chosen by the ``friendly'' player
E, that is, we ask for a weak plan.
With goal $\AA \EE . \varphi$ we require that every plan execution led
by player A can be extended by player E to a successful execution
that satisfies the formula $\varphi$; in the case of a reachability
goal, this corresponds to asking for a strong cyclic solution.
With goal $\EE \AA . \varphi$ we require that, after an initial set of
actions controlled by player E, we have the
guarantee that formula $\varphi$ will be satisfied independently of
how player A will choose the outcome of the following actions.
As a final example, with goal $(\AA \EE)^\omega . \varphi = \AA \EE
\AA \EE \AA \cdots . \varphi$ we require that formula $\varphi$ is
satisfied in all those executions where player E has the
possibility of controlling the action outcome an infinite number of
times.

Path quantifiers can define arbitrary combinations of the
turns of players A and E, and hence different degrees in satisfying
an LTL goal.
We show, however, that, rather surprisingly, only a finite number of
alternatives exist between strong and weak planning: only eight
``canonical'' path quantifiers give rise to plans of different
strength, and every other path quantifier is equivalent to a canonical
one.
The canonical path quantifiers correspond to the games of length one
($\AA$ and $\EE$), two ($\AA\EE$ and $\EE\AA$), and three ($\AA\EE\AA$
and $\EE\AA\EE$), and to the games defining an infinite alternation
between players A and E ($(\AA\EE)^\omega$ and $(\EE\AA)^\omega$).
We also show that, in the case of reachability goals $\varphi = \F
q$, the canonical path quantifiers further collapse.  Only three
different degrees of solution are possible, corresponding to weak
($\EE . \F q$), strong ($\AA . \F q$), and strong cyclic ($\AA\EE . \F
q$) planning.

Finally, we present a planning algorithm for the new goal language and we
study its complexity.
The algorithm is based on an automata-theoretic approach
\cite{EJ88,KVW00}: planning domains and goals are represented as
suitable automata, and planning is reduced to the problem of checking
whether a given automaton is nonempty.
The proposed algorithm has a time complexity that is doubly
exponential in the size of the goal formula.  It is known that the
planning problem is 2EXPTIME-complete for goals of the form $\AA .
\varphi$ \cite{PR90},
and hence the complexity of our algorithm is optimal.

The structure of the paper is as follows.
In Section~\ref{sec:background} we present some preliminaries on
automata theory and on temporal logics.
In Section~\ref{sec:planning} we define planning domains and plans.
In Section~\ref{sec:ae-ltl} we define $\AE$-LTL, our new logic of path
quantifier, and study its basic properties.
In Section~\ref{sec:ae-ltl-planning} we present a planning algorithm
for $\AE$-LTL, while in Section~\ref{sec:reach-and-maintain} we apply
the new logic to the particular cases of reachability and
maintainability goals.
In Section~\ref{sec:conl} we make comparisons with
related works and present some concluding remarks.


\section{Preliminaries}
\label{sec:background}

This section introduces some preliminaries on automata theory and on
temporal logics.


\subsection{Automata Theory}
\label{ssec:aut}

Given a nonempty alphabet $\Sigma$, an infinite word on $\Sigma$ is an
infinite sequence $\sigma_0, \sigma_1, \sigma_2, \ldots$ of symbols
from $\Sigma$.
Finite state automata have been proposed as finite structures that
accept sets of infinite words.
In this paper, we are interested in \emph{tree} automata, namely in
finite state automata that recognize trees on alphabet $\Sigma$,
rather than words.

\begin{definition}[tree]
  \label{d:tree}
  A \emph{(leafless) tree} $\tau$ is a subset of $\Nat^*$ such that:
  \begin{itemize}
   \item $\epsilon \in \tau$ is the root of the tree;
   \item if $x \in \tau$ then there is some $i \in \Nat$ such that $x
    \cdot i \in \tau$;
   \item if $x \cdot i \in \tau$, with $x \in \Nat^*$ and $i \in
    \Nat$, then also $x \in \tau$;
   \item if $x \cdot (i{+}1) \in \tau$, with $x \in \Nat^*$ and $i \in
    \Nat$, then also $x \cdot i \in \tau$.
  \end{itemize}
  The \emph{arity} of $x \in \tau$ is the number of its children,
  namely $\arity(x) = | \{ i \st x \cdot i \in  \tau \} |$.
  Let $\Dcal \subseteq \Nat$. Tree $\tau$ is a \emph{$\Dcal$-tree} if
  $\arity(x) \in \Dcal$ for each $x \in \tau$.
  A \emph{$\Sigma$-labelled tree} is a pair $(\tau, \Tau)$, where
  $\tau$ is a tree and $\Tau: \tau \to \Sigma$.
  In the following, we will denote $\Sigma$-labelled tree $(\tau,
  \Tau)$ as $\Tau$, and let $\tau = \dom(\Tau)$.
\end{definition}
Let $\Tau$ be a $\Sigma$-labelled tree.
A \emph{path} $p$ of $\Tau$ is a (possibly infinite) sequence
$x_0,x_1,\ldots$ of nodes $x_i \in \dom(\Tau)$ such that $x_{k+1}
= x_k \cdot i_{k+1}$.
In the following, we denote with $P^*(\Tau)$ the set of finite paths
and with $P^\omega(\Tau)$ the set of infinite paths of $\Tau$.
Given a (finite or infinite) path $p$, we denote with $\Tau(p)$ the
string $\Tau(x_0) \cdot \Tau(x_1) \cdots$, where $x_0, x_1, \ldots$ is
the sequence of nodes of path $p$.
We say that a finite (resp.\ infinite) path $p'$ is a finite (resp.\ 
infinite) \emph{extension} of the finite path $p$ if the sequence of
nodes of $p$ is a prefix of the sequence of nodes of $p'$.

A tree automaton is an automaton that accepts sets of trees. 
In this paper, we consider a particular family of tree automata,
namely \emph{parity tree automata} \cite{EJ91}. 
\begin{definition}[parity tree automata]
  \label{d:tree-automaton}
  A \emph{parity tree automaton} with parity index $k$ is a tuple $A =
  \langle \Sigma, \Dcal, Q, q_0, \delta, \beta \rangle$, where:
  \begin{itemize}
   \item $\Sigma$ is the finite, nonempty alphabet;
   \item $\Dcal \subseteq \Nat$ is a finite set of arities;
   \item $Q$ is the finite set of states;
   \item $q_0 \in Q$ is the initial state;
   \item $\delta: Q \times \Sigma \times \Dcal \to 2^{Q^*}$ is the
    transition function, where $\delta(q,\sigma,d) \in 2^{Q^d}$;
   \item $\beta: Q \to \{ 0, \ldots, k \}$ is the parity mapping.
  \end{itemize}
\end{definition}
A tree automaton accepts a tree if there is an accepting run of the
automaton on the tree.
Intuitively, when a parity tree automaton is in state $q$ and it is
reading a $d$-ary node of the tree that is labeled by $\sigma$, it
nondeterministically chooses a $d$-tuple $\langle q_1, \ldots, q_{d}
\rangle$ in $\delta(q,\sigma,d)$ and then makes $d$ copies of itself,
one for each child node of the tree, with the state of the $i$-th copy
updated to $q_i$.
A run of the parity tree automaton is accepting if, along every
infinite path, the minimal priority that is visited infinitely often is
an even number.
\begin{definition}[tree acceptance]
  \label{d:tree-acceptance}
  The parity tree automaton $A = \langle \Sigma, \Dcal, Q, q_0,
  \delta, \beta \rangle$ \emph{accepts} the $\Sigma$-labelled
  $\Dcal$-tree $\Tau$ if there exists an \emph{accepting run} $r$ for
  $\Tau$, namely there exists a mapping $r: \tau \to Q$ such that:
  \begin{itemize}
   \item $r(\epsilon) = q_0$;
   \item for each $x \in \tau$ with $\arity(x) = d$ we have \mbox{$\langle
    r(x \cdot 0), \ldots r(x \cdot (d{-}1)) \rangle \in
    \delta(r(x),\Tau(x),d)$};
   \item along every infinite path $x_0, x_1, \ldots$ in $\Tau$, the
    minimal integer $h$ such that $\beta(r(x_i)) = h$ for infinitely many
    nodes $x_i$ is even.
  \end{itemize}
  The tree automaton $A$ is \emph{nonempty} if there exists some tree
  $\Tau$ that is accepted by $A$.
\end{definition}
\citeA{EJ91} have shown that the emptiness of a parity tree
automaton can be decided in a time that is exponential in the parity
index and polynomial in the number of states.
\begin{theorem}
  \label{t:tree-aut-complexity}
  The emptiness of a parity tree automaton with $n$ states and index $k$
  can be determined in time $n^{O(k)}$.
\end{theorem}


\subsection{Temporal Logics}
\label{ssec:logics}

Formulas of \emph{Linear Temporal Logic} (LTL) \cite{Eme90} are built
on top of a set $\mathit{Prop}$ of atomic propositions using the standard
Boolean operators, the unary temporal operator $\X$ (next), and the
binary temporal operator $\U$ (until).
In the following we assume to have a fixed set of atomic propositions
$\mathit{Prop}$, and we define $\Sigma = 2^{\mathit{Prop}}$ as the set of subsets of
$\mathit{Prop}$.

\begin{definition}[LTL]
  \label{d:ltl}
  \LTL formulas $\varphi$ on $\mathit{Prop}$ are defined by the following
  grammar, where $q \in \mathit{Prop}$:
  \begin{equation*}
    \varphi \as q \orbar \neg \varphi \orbar \varphi \wedge \varphi
         \orbar \X \varphi \orbar \varphi \U \varphi
  \end{equation*}
\end{definition}
We define the following auxiliary operators: $\F \varphi = \top \U
\varphi$ (eventually in the future $\varphi$) and $\G \varphi = \neg
\F \neg \varphi$ (always in the future $\varphi$).
LTL formulas are interpreted over infinite words on $\Sigma$. 
In the following, we write $w \models_\LTL \varphi$ whenever the
infinite word $w$ satisfies the \LTL formula $\varphi$.
\begin{definition}[LTL semantics]
  \label{d:ltl-sem}
  Let $w = \sigma_0, \sigma_1, \ldots$ be an infinite word on $\Sigma$
  and let $\varphi$ be an LTL formula. We define $w,i \models_\LTL
  \varphi$, with $i \in \Nat$, as follows:
  \begin{itemize}
   \item $w,i \models_\LTL q$ iff $q \in \sigma_i$;
   \item $w,i \models_\LTL \neg \varphi$ iff it does not hold that $w,i
    \models_\LTL \varphi$;
   \item $w,i \models_\LTL \varphi \wedge \varphi'$ iff $w,i \models_\LTL
    \varphi$ and $w,i \models_\LTL \varphi'$;
   \item $w,i \models_\LTL \X \varphi$ iff $w, i{+}1 \models_\LTL \varphi$;
   \item $w,i \models_\LTL \varphi \U \varphi'$ iff there is some $j \geq i$
    such that $w,k \models_\LTL \varphi$ for all $i \leq k < j$ and $w,j
    \models_\LTL \varphi'$.
  \end{itemize}
  We say that $w$ satisfies $\varphi$, written $w \models_\LTL \varphi$, if
  $w,0 \models_\LTL \varphi$.
\end{definition}

\CTLstar \cite{Eme90} is an example of ``branching-time'' logic. Path
quantifiers $\Ap$ (``for all paths'') and $\Ep$ (``for some path'') can
prefix arbitrary combinations of linear time operators.
\begin{definition}[\CTLstar]
  \label{d:ctlstar}
  \CTLstar formulas $\psi$ on $\mathit{Prop}$ are defined by the following
  grammar, where $q \in \mathit{Prop}$:
  \begin{eqnarray*}
    \psi    & \as & q \orbar \neg \psi \orbar \psi \wedge \psi \orbar
                    \Ap \varphi \orbar \Ep \varphi\\
    \varphi & \as & \psi \orbar \neg \varphi \orbar
                    \varphi \wedge \varphi \orbar
                    \X \varphi \orbar \varphi \U \varphi
  \end{eqnarray*}
\end{definition}
\CTLstar formulas are interpreted over $\Sigma$-labelled trees.
In the following, we write $\Tau \models_\CTLstar \psi$ whenever
$\Tau$ satisfies the \CTLstar formula $\psi$.
\begin{definition}[\CTLstar semantics]
  \label{d:ctlstar-sem}
  Let $\Tau$ be a $\Sigma$-labelled tree and let $\psi$ be a \CTLstar
  formula. We define $\Tau,x \models_\CTLstar \psi$, with $x \in
  \tau$, as follows:
  \begin{itemize}
   \item $\Tau,x \models_\CTLstar q$ iff $q \in \Tau(x)$;
   \item $\Tau,x \models_\CTLstar \neg \psi$ iff it does not hold
    that $\Tau,x \models_\CTLstar \psi$;
   \item $\Tau,x \models_\CTLstar \psi \wedge \psi'$ iff $\Tau,x
    \models_\CTLstar \psi$ and $\Tau,x \models_\CTLstar \psi'$;
   \item $\Tau,x \models_\CTLstar \Ap \varphi$ iff $\Tau,p \models_\CTLstar
    \varphi$ holds for all infinite paths $p = x_0, x_1, \ldots$ with
    $x_0 = x$;
   \item $\Tau,x \models_\CTLstar \Ep \varphi$ iff $\Tau,p \models_\CTLstar
    \varphi$ holds for some infinite path $p = x_0, x_1, \ldots$ with
    $x_0 = x$;
  \end{itemize}
  where $\Tau,p \models_\CTLstar \phi$, with $p \in P^\omega(\Tau)$,
  is defined as follows:
  \begin{itemize}
   \item $\Tau,p \models_\CTLstar \psi$ iff $p = x_0, x_1, \ldots$
    and $\Tau,x_0 \models_\CTLstar \psi$;
   \item $\Tau,p \models_\CTLstar \neg \varphi$ iff it does not hold
    that $\Tau,p \models_\CTLstar \varphi$;
   \item $\Tau,p \models_\CTLstar \varphi \wedge \varphi'$ iff $\Tau,p
    \models_\CTLstar \varphi$ and $\Tau,p \models_\CTLstar \varphi'$;
   \item $\Tau,p \models_\CTLstar \X \varphi$ iff $\Tau,p'
    \models_\CTLstar \varphi$, where $p' = x_1, x_2, \ldots$ if $p =
    x_0, x_1, x_2, \ldots$;
   \item $\Tau,p \models_\CTLstar \varphi \U \varphi'$ iff there is
    some $j \geq 0$ such that $\Tau,p_k \models_\CTLstar \varphi$ for
    all $0 \leq k < j$ and $\Tau,p_j \models_\CTLstar \varphi'$, where
    $p_i = x_i, x_{i+1}, \ldots$ if $p = x_0, x_1, \ldots$.
  \end{itemize}
  We say that $\Tau$ satisfies the \CTLstar formula $\psi$, written
  $\Tau \models_\CTLstar \psi$, if $\Tau, \epsilon \models_\CTLstar
  \psi$.
\end{definition}

The following theorem states that it is possible to build a tree
automaton that accepts all the trees satisfying a \CTLstar formula.
The tree automaton has a number of states that is doubly exponential and a
parity index that is exponential in the length of the formula. A proof
of this theorem has been given by \citeA{EJ88}.
\begin{theorem}
  \label{t:ctlstar-tree-aut}
  Let $\psi$ be a \CTLstar formula, and let $\Dcal \subseteq \Nat^*$
  be a finite set of arities.
  One can build a parity tree automaton $A^\Dcal_\psi$ that accepts
  exactly the $\Sigma$-labelled $\Dcal$-trees that satisfy
  $\psi$.
  The automaton $A^\Dcal_\psi$ has $2^{2^{O(|\psi|)}}$ states and parity
  index $2^{O(|\psi|)}$, where $|\psi|$ is the length of formula
  $\psi$.
\end{theorem}



\section{Planning Domains and Plans}
\label{sec:planning}

A \emph{(nondeterministic) planning domain} \cite{CPRT03} can be
expressed in terms of a set of \emph{states}, one of which is
designated as the \emph{initial state}, a set of \emph{actions},
and a \emph{transition function} describing how (the execution of)
an action leads from one state to possibly many different states.
\begin{definition}[planning domain]
  \label{d:domain}
  A \emph{planning domain} is a tuple $D = \langle \Sigma,\sigma_0,A,R
  \rangle$ where:
  \begin{itemize}
   \item $\Sigma$ is the finite set of states;
   \item $\sigma_0 \in \Sigma$ is the initial state;
   \item $A$ is the finite set of actions;
   \item $R: \Sigma \times A \to 2^\Sigma$ is the transition relation.
  \end{itemize}
  We require that for each $\sigma \in \Sigma$ there is some $a \in
  A$ and some $\sigma' \in \Sigma$ such that $\sigma' \in R(\sigma,
  a)$.
  We assume that states $\Sigma$ are ordered, and we write
  $R(\sigma,a) = \langle \sigma_1, \sigma_2, \ldots, \sigma_n \rangle$
  whenever $R(\sigma,a) = \{ \sigma_1, \sigma_2, \ldots, \sigma_n \}$
  and $\sigma_1 < \sigma_2 < \cdots < \sigma_n$.
\end{definition}

\begin{figure}
\includegraphics{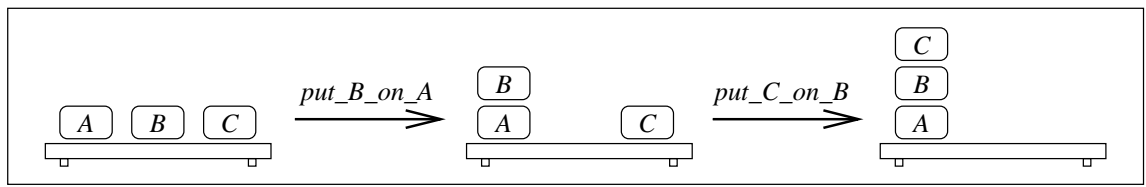}
\caption{A possible scenario in the blocks-world domain.}
\label{fig:bw-exa}
\end{figure}

\begin{example}
  \label{exa:bw}
  Consider a blocks-world domain consisting of a set of blocks, which
  are initially on a table, and which can be stacked on top of each
  other in order to build towers (see Figure~\ref{fig:bw-exa}).

  The states $\Sigma$ of this domain are the possible configurations
  of the blocks: in the case of three blocks there are 13 states,
  corresponding to all the blocks on the table (1 configuration), a
  2-block tower and the remaining block on the table (6
  configurations), and a 3-block tower (6 possible configurations).
  We assume that initially all blocks are on the table.

  The actions in this domain are $\mathit{put\_X\_on\_Y}$,
  $\mathit{put\_X\_on\_table}$, and $\mathit{wait}$, where $X$ and $Y$
  are two (different) blocks.
  Actions $\mathit{put\_X\_on\_Y}$ and $\mathit{put\_X\_on\_table}$
  are possible only if there are no blocks on top of $X$ (otherwise we
  could not pick up $X$). In addition, action $\mathit{put\_X\_on\_Y}$
  requires that there are no blocks on top of $Y$ (otherwise we could
  not put $X$ on top of $Y$).

  We assume that the outcome of action $\mathit{put\_X\_on\_Y}$ is
  nondeterministic: indeed, trying to put a block on top of a tower
  may fail, in which case the tower is destroyed.
  Also action $\mathit{wait}$ is nondeterministic: it is possible that
  the table is bumped and that all its towers are destroyed.
\end{example}

A \emph{plan} guides the evolution of a planning domain by issuing
actions to be executed.
In the case of nondeterministic domains, \emph{conditional plans}
\cite{CPRT03,PT01} are required, that is, the next action issued by the
plan may depend on the outcome of the previous actions.
Here we consider a very general definition of plans: a plan is a
mapping from a sequence of states, representing the past history of
the domain evolution, to an action to be executed.
\begin{definition}[plan]
  \label{d:plan}
  A \emph{plan} is a partial function $\pi : \Sigma^+ \partialto A$ such that:
  \begin{itemize}
   \item if $\pi(w \cdot \sigma) = a$, then $\sigma' \in R(\sigma,a)$ for
    some $\sigma'$;
   \item if $\pi(w \cdot \sigma) = a$, then $\sigma' \in R(\sigma,a)$ iff
    $w \cdot \sigma \cdot \sigma' \in \dom(\pi)$;
   \item if $w \cdot \sigma \in \dom(\pi)$ with $w \neq \epsilon$,
    then $w \in \dom(\pi)$;
   \item $\pi(\sigma)$ is defined iff $\sigma = \sigma_0$ is the
    initial state of the domain.
  \end{itemize}
\end{definition}
The conditions in the previous definition ensure that a plan defines
an action to be executed for exactly the finite paths $w \in
\Sigma^+$ that can be reached executing the plan from the initial
state of the domain.

\begin{example}
  \label{exa:bw-plan}
  A possible plan for the blocks-world domain of Example~\ref{exa:bw}
  is represented in Figure~\ref{fig:bw-plan}.
  We remark the importance of having plans in which the action to be
  executed depends on the whole sequence of states corresponding to
  the past history of the evolution.
  Indeed, according to the plan if Figure~\ref{fig:bw-plan}, two
  different actions $\mathit{put\_C\_on\_A}$ and
  $\mathit{put\_C\_on\_table}$ are performed in the state with block
  $B$ on top of $A$, depending on the past history.
\end{example}
\begin{figure}
  \def\FA{\framebox{\hbox to 1.2cm{\vbox to 1cm {\vfill\hbox{A}}\ {B}\ {C} \hfill}}}
  \def\FB{\framebox{\hbox to 1.2cm{\vbox to 1cm {\vfill\hbox{B}\vspace{2pt}\hbox{A}}\ {C} \hfill}}}
  \def\FC{\framebox{\hbox to 1.2cm{\vbox to 1cm {\vfill\hbox{C}\vspace{2pt}\hbox{B}\vspace{2pt}\hbox{A}} \hfill}}}
  \begin{center}
    \begin{tabular}{|l|l|}
      \hline
      \hfill $w$ \hfill\mbox{} & $\hfill \pi(w)$ \hfill \mbox{}\\[2pt]\hline 
      \ & \ \\[-3mm]
      $\FA$ & $\mathit{put\_B\_on\_A}$ \\[2pt]\hline
      \ & \ \\[-3mm]
      $\FA\cdot\FB$ & $\mathit{put\_C\_on\_B}$ \\[2pt]\hline
      \ & \ \\[-3mm]
      $\FA\cdot\FB\cdot\FC$ & $\mathit{put\_C\_on\_table}$ \\[2pt]\hline
      \ & \ \\[-3mm]
      $\FA\cdot\FB\cdot\FC\cdot\FB$ & $\mathit{put\_B\_on\_table}$ \\[2pt]\hline
      \ & \ \\[-3mm]
      $\FA\cdot\FB\cdot\FC\cdot\FB\cdot\FA$ &$\mathit{wait}$\\[2pt]\hline
      \ & \ \\
      \hfill \textit{any other history} \hfill\mbox{}&$\mathit{wait}$\\[2pt]\hline
    \end{tabular}
 \end{center}  
 \caption{A plan for the blocks-world domain.}
 \label{fig:bw-plan}
\end{figure}

Since we consider nondeterministic planning domains, the execution of
an action may lead to different outcomes. Therefore, the execution of
a plan on a planning domain can be described as a
$(\Sigma{\times}A)$-labelled tree. Component $\Sigma$ of the label of
the tree corresponds to a state in the planning domain, while
component $A$ describes the action to be executed in that state.

\begin{definition}[execution tree]
  \label{d:execution-tree}
  The \emph{execution tree} for domain $D$ and plan $\pi$ is the
  $(\Sigma{\times}A)$-labelled tree $\Tau$ defined as follows:
  \begin{itemize}
   \item $\Tau(\epsilon) = (\sigma_0,a_0)$ where $\sigma_0$ is the
    initial state of the domain and $a_0 = \pi(\sigma_0)$;
   \item if $p = x_0, \ldots, x_n \in P^*(\Tau)$ with $\Tau(p) =
    (\sigma_0,a_0) \cdot (\sigma_1,a_1) \cdots (\sigma_n,a_n)$, and if
    $R(\sigma_n,a_n) = \langle \sigma'_0, \ldots, \sigma'_{d-1}
    \rangle$, then for every $0 \leq i < d$ the following conditions
    hold:
    $x_n \cdot i \in \dom(\Tau)$ and $\Tau(x_n \cdot i)
    = (\sigma'_i,a'_i)$ with $a'_i = \pi(\sigma_0 \cdot \sigma_1 \cdots
    \sigma_n \cdot \sigma'_i)$.
  \end{itemize}
\end{definition}

A \emph{planning problem} consists of a planning domain and of a goal
$g$ that defines the set of desired behaviors. In the following, we
assume that the goal $g$ defines a set of execution trees, namely the
execution trees that exhibit the behaviors described by the goal (we
say that these execution trees satisfy the goal).

\begin{definition}[planning problem]
  \label{d:planning-problem}
  A \emph{planning problem} is a pair $(D,g)$, where $D$ is a
  planning domain and $g$ is a \emph{goal}.
  A \emph{solution} to a planning problem $(D,g)$ is a plan $\pi$ such
  that the execution tree for $\pi$ satisfies the goal $g$.
\end{definition}


\section{A Logic of Path Quantifiers}
\label{sec:ae-ltl}

In this section we define a new logic that is based on LTL and that
extends it with the possibility of defining conditions on the sets of
paths that satisfy the LTL property.
We start by motivating why such a logic is necessary for defining
planning goals.

\begin{example}
  \label{exa:ltl-goals}
  Consider the blocks-world domain introduced in the previous section.
  Intuitively, the plan of Example~\ref{exa:bw-plan} is a solution to
  the goal of building a tower consisting of blocks $A$, $B$, $C$
  and then of destroying it. This goal can be easily formulated as an
  LTL formula:
  \begin{equation*}
    \varphi_1 = \F\ ((\mathit{C\_on\_B} \wedge \mathit {B\_on\_A} \wedge \mathit{A\_on\_table}) \wedge
    \F\ (\mathit{C\_on\_table} \wedge \mathit {B\_on\_table} \wedge \mathit{A\_on\_table})).
  \end{equation*}
  Notice however that, due to the nondeterminism in the outcome of
  actions, this plan may fail to satisfy the goal.  It is possible,
  for instance, that action $\mathit{put\_C\_on\_B}$ fails and the
  tower is destroyed. In this case, the plan proceeds performing
  $\mathit{wait}$ actions, and hence the tower is never finished.
  Formally, the plan is a solution to the goal which requires that
  there is \emph{some} path in the execution structure that satisfies
  the \LTL formula $\varphi_1$.

  Clearly, there are better ways to achieve the goal of building a
  tower and then destroying it: if we fail building the tower, rather
  than giving up, we can restart building it and keep trying until we
  succeed.
  This strategy allows for achieving the goal in ``most of the
  paths'': only if we keep destroying the tower when we try to build
  it we will not achieve the goal. As we will see,
  the logic of path quantifiers that
  we are going to define will allow us to formalize what we mean by
  ``most of the paths''.
  
  Consider now the following LTL formula:
  \begin{equation*}
    \varphi_2 = \F\,\G\ ((\mathit{C\_on\_B} \wedge \mathit {B\_on\_A} \wedge \mathit{A\_on\_table}).
  \end{equation*}
  The formula requires building a tower and maintaining it.
  In this case we have two possible ways to fail to achieve the goal.
  We can fail to build the tower; or, once built, we can fail to
  maintain it (remember that a $\mathit{wait}$ action may nondeterministically
  lead to a destruction of the tower).
  Similarly to the case of formula $\phi_1$, a planning goal that
  requires satisfying the formula $\phi_2$ in \emph{all} paths of the
  execution tree is unsatisfiable. On the other hand, a goal that
  requires satisfying it on \emph{some} paths is very weak; our logic
  allows us to be more demanding on the paths that satisfy the formula.

  Finally, consider the following LTL formula:
  \begin{equation*}
    \varphi_3 = \G\,\F\ ((\mathit{C\_on\_B} \wedge \mathit {B\_on\_A} \wedge \mathit{A\_on\_table}).
  \end{equation*}
  It requires that the tower exists infinitely many time, i.e., if the
  tower gets destroyed, then we have to rebuild it.
  Intuitively, this goal admits plans that can achieve it more often,
  i.e., on ``more paths'', than $\varphi_2$. Once again, a path logic
  is needed to give a formal meaning to ``more paths''.
\end{example}

In order to be able to represent the planning goals discussed
in the previous example, we consider logic formulas of the form $\alpha .
\varphi$, where $\varphi$ is an LTL formula and $\alpha$ is a path
quantifier and defines a set of infinite paths on which the formula
$\varphi$ should be checked.
Two extreme cases are the path quantifier $\AA$, which is used to denote
that $\varphi$ must hold on \emph{all} the paths, and the path
quantifier $\EE$, which is used to denote that $\varphi$ must hold on
\emph{some} paths.
In general, a path quantifier is a (finite or infinite) word on
alphabet $\{ {\AA}, {\EE} \}$ and defines an alternation in the
selection of the two modalities corresponding to $\EE$ and $\AA$.
For instance, by writing $\AA \EE . \varphi$ we require that all finite
paths have some infinite extension that satisfies $\varphi$, while by
writing $\EE \AA . \varphi$ we require that all the extensions of some
finite path satisfy $\varphi$.

The path quantifier can be seen as the definition of a two-player
game for the selection of the paths that should satisfy the LTL
formula.
Player A (corresponding to $\AA$) tries to build a path that does not
satisfy the LTL formula, while player E (corresponding to $\EE$) tries
to build the path so that the LTL formula holds.
Different path quantifiers define different alternations in the turns
of players A and E.
The game starts from the path consisting only of the initial state,
and, during their turns, players A and E extend the path by a finite
number of nodes. In the case the path quantifier is a finite word, the
player that moves last in the game extends the finite path built so
far to an infinite path.
The formula is satisfied if player E has a winning
strategy, namely if, for all the possible moves of the player A, it is
always able to build a path that satisfies the LTL formula.

\begin{example}
  Let us consider the three LTL formulas defined in
  Example~\ref{exa:ltl-goals}, and let us see how the path quantifiers
  we just introduced can be applied.

  In the case of formula $\varphi_1$, the plan presented in
  Example~\ref{exa:bw-plan} satisfies requirement $\EE. \varphi_1$:
  there is a path on which the tower is built and then destroyed.  It
  also satisfies the ``stronger'' requirement $\EE\AA. \varphi_1$
  that stresses the fact that, in this case, once the tower has been
  built and destroyed, we can safely give the control to player A.
  Formula $\varphi_1$ can be satisfied in a stronger way, however.
  Indeed, the plan that keeps trying to build the tower satisfies the
  requirement $\AA\EE. \varphi_1$, as well as the requirement
  $\AA\EE\AA. \varphi_1$: player A cannot
  reach a state where the satisfaction of the goal is prevented.

  Let us now consider the formula $\varphi_2$. In this case, we can find
  plans satisfying $\AA\EE. \varphi_2$, but no plan can satisfy
  requirement $\AA\EE\AA. \varphi_2$. Indeed, player A has a simple
  strategy to win, if he gets the control after we built the tower:
  bump the table.
  Similar considerations hold also for formula $\varphi_3$. Also in
  this case, we can find plans for requirement $\AA\EE. \varphi_3$,
  but not for requirement $\AA\EE\AA. \varphi_3$. In this case,
  however, plans exist also for requirement $\AA\EE\AA\EE\AA\EE\cdots.
  \varphi_3$: if player E gets the control infinitely often, then it
  can rebuild the tower if needed.
\end{example}

In the rest of the section we give a formal definition and study the
basic properties of this logic of path quantifiers.


\subsection{Finite Games}
\label{ssec:ae-ltl-finite}

We start considering only games with a finite number of moves, that is
path quantifiers corresponding to finite words on $\{ {\AA}, {\EE} \}$.

\begin{definition}[$\AE$-LTL]
  \label{d:ae-ltl}
  An $\AE$-LTL formula is a pair $g = \alpha . \varphi$, where
  $\varphi$ is an LTL formula and $\alpha \in \{ \AA, \EE \}^+$ is a
  path quantifier.
\end{definition}

The following definition describes the games corresponding to the
finite path quantifiers.
\begin{definition}[semantics of $\AE$-LTL]
  \label{d:ae-ltl-sem}
  Let $p$ be a finite path of a $\Sigma$-labelled tree $\Tau$. Then:
  \begin{itemize}
   \item $p \models \AA \alpha . \varphi$ if for all finite extensions
    $p'$ of $p$ it holds that $p' \models \alpha . \varphi$.
   \item $p \models \EE \alpha . \varphi$ if for some finite extension
    $p'$ of $p$ it holds that $p' \models \alpha . \varphi$.
   \item $p \models \AA . \varphi$ if for all infinite extensions $p'$ of
    $p$ it holds that $\Tau(p') \models_\LTL \varphi$.
   \item $p \models \EE . \varphi$ if for some infinite extension $p'$ of
    $p$ it holds that $\Tau(p') \models_\LTL \varphi$.
  \end{itemize}

  \noindent
  We say that the $\Sigma$-labelled tree $\Tau$ satisfies the
  $\AE$-LTL formula $g$, and we write $\Tau \models g$, if 
  $p_0 \models g$, where $p_0 = \epsilon$ is the root of
  $\Tau$.
\end{definition}

$\AE$-LTL allows for path quantifiers consisting of an arbitrary
combination of $\AA$s and $\EE$s.
Each combination corresponds to a different set of rules for the game
between A and E.
In Theorem~\ref{t:123} we show that all this freedom in the
definition of the path quantifier is not needed.
Only six path quantifiers are sufficient to capture all the possible
games.
This result is based on the concept of \emph{equivalent path
quantifiers}.  

Consider formulas $\AA . \F p$ and $\AA \EE . \F p$. It is easy to see
that the two formulas are equi-satisfiable, i.e., if a tree $\Tau$
satisfies $\AA . \F p$ then it also satisfies $\AA \EE. \F p$, and
vice-versa.
In this case, path quantifiers $\AA$ and $\AA\EE$ have the same
``power'', but this depends on the fact that we use the path
quantifiers in combination with the LTL
formula $\F p$.  If we combine the two path quantifiers with different
LTL formulas, such as $\G p$, it is possible to find trees that
satisfy the latter path quantifier but not the former.  For this
reason, we cannot consider the two path quantifiers equivalent.
Indeed, in order for two path quantifiers to be equivalent, they have
to be equi-satisfiable for \emph{all} the LTL formulas. This intuition
is formalized in the following definition.

\begin{definition}[equivalent path quantifiers]
  \label{d:strictness}
  Let $\alpha$ and $\alpha'$ be two path quantifiers.
  We say that $\alpha$ \emph{implies} $\alpha'$, written $\alpha
  \aimply \alpha'$, if for all $\Sigma$-labelled trees $\Tau$ and
  for all LTL formulas $\varphi$, $\Tau \models \alpha .  \varphi$
  implies $\Tau \models \alpha' . \varphi$.
  We say that $\alpha$ is \emph{equivalent} to $\alpha'$, written $\alpha
  \aequiv \alpha'$, if $\alpha \aimply \alpha'$ and $\alpha' \aimply
  \alpha$.
\end{definition}

The following lemma describes some basic properties of path
quantifiers and of the equivalences among them. We will exploit these
results in the proof of Theorem~\ref{t:123}.

\begin{lemma}
  \label{l:ae-ltl-prop}
  Let $\alpha, \alpha' \in \{\AA, \EE\}^*$. The following
  implications and equivalences hold.
  \begin{enumerate}
   \item $\alpha \AA \AA \alpha' \aequiv \alpha \AA \alpha'$ and
    $\alpha \EE \EE \alpha' \aequiv \alpha \EE \alpha'$.
   \item $\alpha \AA \alpha' \aimply \alpha \alpha'$ and $\alpha
    \alpha' \aimply \alpha \EE \alpha'$, if $\alpha\alpha'$ is not empty.
   \item $\alpha \AA \alpha' \aimply \alpha \AA \EE \AA \alpha'$ and
    $\alpha \EE \AA \EE \alpha' \aimply \alpha \EE \alpha'$.
   \item $\alpha \AA \EE \AA \EE \alpha' \aequiv \alpha \AA \EE
    \alpha'$ and $\alpha \EE \AA \EE \AA \alpha' \aequiv \alpha \EE
    \AA \alpha'$.
  \end{enumerate}
\end{lemma}
\begin{proof}
  In the proof of this lemma, in order to prove that $\alpha\alpha'
  \aimply \alpha\alpha''$ we prove that, given an arbitrary
  tree $\Tau$ and an arbitrary LTL formula $\varphi$, $p
  \models \alpha' . \varphi$ implies $p \models \alpha'' . \varphi$
  for every finite path $p$ of $\Tau$.
  Indeed, if $p \models \alpha' . \varphi$ implies $p
  \models \alpha'' . \varphi$ for all finite paths $p$, then it is
  easy to prove, by induction on $\alpha$, that $p \models \alpha
  \alpha' .  \varphi$ implies $p \models \alpha \alpha'' . \varphi$
  for all finite paths $p$.
  In the following, we will refer to this proof technique as prefix
  induction.
  
  \begin{enumerate}
   \item We show that, for every finite path $p$, $p \models \AA\AA
    \alpha' . \varphi$ if and only if $p \models \AA \alpha' .
    \varphi$: then the equivalence of $\alpha \AA\AA \alpha'$ and
    $\alpha \AA \alpha'$ follows by prefix induction.
    
    Let us assume that $p \models \AA\AA \alpha' . \varphi$.  We prove
    that $p \models \AA \alpha' . \varphi$, that is, that $p' \models
    \alpha' . \varphi$ for every finite\footnote{We assume that
    $\alpha'$ is not the empty word. The proof in the case $\alpha'$
    is the empty word is similar.} extension $p'$ of $p$.
    Since $p \models \AA \AA \alpha' . \varphi$, by
    Definition~\ref{d:ae-ltl-sem} we know that, for every finite
    extension $p'$ of $p$, $p' \models \AA \alpha' .  \varphi$.
    Hence, again by Definition~\ref{d:ae-ltl-sem}, we know that for
    every finite extension $p''$ of $p'$, $p'' \models \alpha' .
    \varphi$.
    Since $p'$ is a finite extension of $p'$, we can conclude that $p'
    \models \alpha' . \varphi$. Therefore, $p' \models \alpha' .
    \varphi$ holds for all finite extensions $p'$ of $p$.

    Let us now assume that $p \models \AA \alpha' . \varphi$.
    We prove that $p \models \AA\AA \alpha' . \varphi$, that is,
    for all finite extensions $p'$ of $p$, and for all finite
    extensions $p''$ of $p'$, $p'' \models \alpha' . \varphi$.
    We remark that the finite path $p''$ is also a finite extension of
    $p$, and therefore $p'' \models \alpha' . \varphi$ holds since $p
    \models \AA \alpha' . \varphi$.
    
    This concludes the proof of the equivalence of $\alpha \AA \AA
    \alpha'$ and $\alpha \AA \alpha'$.
    The proof of the equivalence of $\alpha \EE \EE \alpha'$ and
    $\alpha \EE \alpha'$ is similar.
    
   \item Let us assume first that $\alpha'$ is not an empty word.
    We distinguish two cases, depending on the first symbol of
    $\alpha'$.
    If $\alpha' = \AA \alpha''$, then we should prove that
    $\alpha \AA \AA \alpha'' \aimply \alpha \AA \alpha''$,
    which we already did in item 1 of this lemma.
    If $\alpha' = \EE \alpha''$, then we show that, for every finite
    path $p$, if $p \models \AA \EE \alpha'' . \varphi$ then $p
    \models \EE \alpha'' .  \varphi$: then $\alpha \AA \alpha' \aimply
    \alpha \alpha'$ follows by prefix induction.
    Let us assume that $p \models \AA \EE \alpha'' . \varphi$.
    Then, for all finite extensions $p'$ of $p$ there exists some
    finite\footnote{We assume that $\alpha''$ is
    not the empty word. The proof in the case where $\alpha''$ is empty is
    similar.} extension $p''$ of $p'$ such that $p'' \models \alpha' .
    \varphi$.
    Let us take $p' = p$. Then we know that there is some finite
    extension $p''$ of $p$ such that $p'' \models \alpha' . \varphi$,
    that is, according to Definition~\ref{d:ae-ltl-sem}, $p \models
    \EE \alpha' . \varphi$.

    Let us now assume that $\alpha'$ is the empty word. By hypothesis,
    $\alpha\alpha' \neq \epsilon$, so $\alpha$ is not empty.
    We distinguish two cases, depending on the
    last symbol of $\alpha$.
    If $\alpha = \alpha'' \AA$, then we should prove that $\alpha''
    \AA \AA \aimply \alpha'' \AA$, which we already did in item 1
    of this lemma.
    If $\alpha = \alpha'' \EE$, then we prove that for every finite
    path $p$, if $p \models \EE \AA . \varphi$ then $p \models \EE .
    \varphi$: then $\alpha'' \EE \AA \aimply \alpha'' \EE$ follows by
    prefix induction.
    Let us assume that $p \models \EE \AA . \varphi$. By
    Definition~\ref{d:ae-ltl-sem}, there exists some finite extension
    $p'$ of $p$ such that, for every infinite extension $p''$ of $p'$
    we have $\Tau(p'') \models_\LTL \varphi$.
    Let $p''$ be any infinite extension of $p'$. We know that $p''$ is
    also an infinite extension of $p$, and that $\Tau(p'')
    \models_\LTL \varphi$. Then, by Definition~\ref{d:ae-ltl-sem} we
    deduce that $p \models \EE .  \varphi$.
    
    This concludes the proof that $\alpha \AA \alpha' \aimply \alpha
    \alpha'$.
    The proof that $\alpha \alpha' \aimply \alpha \EE \alpha'$ is
    similar.
    
   \item By item 1 of this lemma we know that $\alpha \AA \alpha'
    \aimply \alpha \AA \AA \alpha'$ and by item 2 we know that $\alpha
    \AA \AA \alpha' \aimply \alpha \AA \EE \AA \alpha'$.
    This concludes the proof that $\alpha \AA \alpha' \aimply \alpha
    \AA \EE \AA \alpha'$.
    The proof that $\alpha \EE \AA \EE \alpha' \aimply \alpha \EE
    \alpha'$ is similar.
    
   \item By item 3 of this lemma we know that $(\alpha \AA) \EE \AA
    \EE \alpha' \aimply (\alpha \AA) \EE \alpha'$.
    Moreover, again by item 3, we know that $\alpha \AA (\EE \alpha')
    \aimply \alpha \AA \EE \AA (\EE \alpha')$.
    Therefore, we deduce $\alpha \AA \EE \alpha' \aequiv \alpha \AA
    \EE \AA \EE \alpha'$.
    The proof that $\alpha \EE \AA \alpha' \aequiv \alpha \EE \AA \EE
    \AA \alpha'$ is similar.
    \qed
  \end{enumerate}
\end{proof}

We can now prove the first main result of the paper: each finite path
quantifier is equivalent to a \emph{canonical path quantifier} of
length at most three.
\begin{theorem}
  \label{t:123}
  For each finite path quantifier $\alpha$ there is a canonical finite
  path quantifier $$\alpha' \in \{ \AA, \EE, \AA\EE, \EE\AA, \AA\EE\AA,
  \EE\AA\EE \}$$ such that $\alpha \aequiv \alpha'$.
  Moreover, the following implications hold between the
  canonical finite path quantifiers:
  \begin{equation}
    \label{eq:123}
    \begin{xy}
      \xymatrix{
        \AA \ar@{~>}[r] &
        \AA\EE\AA \ar@{~>}[r]\ar@{~>}[d] &
        \AA\EE \ar@{~>}[d] \\
        &
        \EE\AA \ar@{~>}[r] &
        \EE\AA\EE \ar@{~>}[r] &
        \EE
      }
    \end{xy}
  \end{equation}
\end{theorem}
\begin{proof}
  We first prove that each path quantifier $\alpha$ is equivalent to
  some canonical path quantifier $\alpha'$.
  By an iterative application of Lemma~\ref{l:ae-ltl-prop}(1), we
  obtain from $\alpha$ a path quantifier $\alpha''$ such that $\alpha
  \aequiv \alpha''$ and $\alpha''$ does not contain two adjacent $\AA$
  or $\EE$.
  Then, by an iterative application of Lemma~\ref{l:ae-ltl-prop}(4),
  we can transform $\alpha''$ into an equivalent path quantifier
  $\alpha'$ of length at most 3.
  The canonical path quantifiers in (\ref{eq:123}) are precisely those
  quantifiers of length at most 3 that do not contain two adjacent
  $\AA$ or $\EE$.

  For the implications in (\ref{eq:123}):
  \begin{itemize}
   \item $\AA \aimply \AA\EE\AA$ and $\EE\AA\EE \aimply \EE$
    come from Lemma~\ref{l:ae-ltl-prop}(3);
   \item $\AA\EE\AA \aimply \EE\AA$ and $\AA\EE \aimply \EE\AA\EE$
    come from Lemma~\ref{l:ae-ltl-prop}(2);
   \item $\AA\EE\AA \aimply \AA\EE$ and $\EE\AA \aimply \EE\AA\EE$
    come from Lemma~\ref{l:ae-ltl-prop}(2).
    \qed
  \end{itemize}
\end{proof}

We remark that Lemma~\ref{l:ae-ltl-prop} and Theorem~\ref{t:123} do
not depend on the usage of LTL for formula $\varphi$.
They depend on the general observation that $\alpha \aimply \alpha'$
whenever player E can select for game $\alpha'$ a set of paths which
is a subset of those selected for game $\alpha$.


\subsection{Infinite Games}
\label{ssec:ae-ltl-infinite}

We now consider infinite games, namely path quantifiers consisting of
infinite words on alphabet $\{ {\AA}, {\EE} \}$.
We will see that infinite games can express all the finite path
quantifiers that we have studied in the previous subsection, but that
there are some infinite games, corresponding to an infinite
alternation of the two players A and E, which cannot be expressed with
finite path quantifiers.

In the case of infinite games, we assume that player E moves
according to a strategy $\xi$ that suggests how to extend each finite
path.
We say that $\Tau \models \alpha . \varphi$, where $\alpha$ is an
infinite game, if there is some winning strategy $\xi$ for player E.
A strategy $\xi$ is winning if, whenever $p$ is an infinite path of
$\Tau$ obtained according to $\alpha$ --- 
i.e., by allowing player A to play in an arbitrary way and by
requiring that player E follows strategy $\xi$ --- then $p$ satisfies the
LTL formula $\varphi$.

\begin{definition}[strategy]
  \label{d:strategy}
  A strategy for a $\Sigma$-labelled tree $\Tau$ is a mapping $\xi:
  P^*(\Tau) \to P^*(\Tau)$ that maps every finite path $p$ to one of
  its finite extensions $\xi(p)$.
\end{definition}

\begin{definition}[semantics of $\AE$-LTL]
  \label{d:ae-ltl-inf-sem}
  Let $\alpha = \Pi_0 \Pi_1 \cdots$ with $\Pi_i \in
  \{ {\AA}, {\EE} \}$ be an infinite path quantifier.
  An infinite path $p$ is a \emph{possible outcome} of game $\alpha$
  with strategy $\xi$ if there is a \emph{generating sequence} for it,
  namely, an infinite sequence $p_0, p_1, \ldots$ of finite paths such
  that:
  \begin{itemize}
   \item $p_i$ are finite prefixes of $p$;
   \item $p_0 = \epsilon$ is the root of tree $\Tau$;
   \item if $\Pi_i = \EE$ then $p_{i+1} = \xi(p_i)$;
   \item if $\Pi_i = \AA$ then $p_{i+1}$ is an (arbitrary) extension of
    $p_i$.
  \end{itemize}
  We denote with $\Out_\Tau(\alpha, \xi)$ the set of infinite paths of
  $\Tau$ that are possible outcomes of game $\alpha$ with strategy
  $\xi$.
  The tree $\Tau$ satisfies the $\AE$-LTL formula $g = \alpha .
  \varphi$, written $\Tau \models g$, if there is some strategy $\xi$
  such that $\Tau(p) \models_\LTL \varphi$ for all paths $p \in
  \Out_\Tau(\alpha, \xi)$.
\end{definition}

We remark that it is possible that the paths in a generating sequence
stop growing, i.e., that there is some $p_i$ such that $p_i = p_j$ for
all $j \geq i$.
In this case, according to the previous definition, all infinite paths
$p$ that extend $p_i$ are possible outcomes.

In the next lemmas we extend the
analysis of equivalence among path quantifiers
to infinite games.\footnote{The
definitions of the implication and equivalence relations
(Definition~\ref{d:strictness}) also apply to the case of infinite
path quantifiers.}
The first lemma shows that finite path quantifiers are just particular
cases of infinite path quantifiers, namely, they correspond to those
infinite path quantifiers that end with an infinite sequence of $\AA$
or of $\EE$.

\begin{lemma}
  \label{l:ae-ltl-inf-fin}
  Let $\alpha$ be a finite path quantifier. Then $\alpha (\AA)^\omega
  \aequiv \alpha \AA$ and $\alpha (\EE)^\omega \aequiv \alpha \EE$.
\end{lemma}
\begin{proof}
  We prove that $\alpha (\AA)^\omega \aequiv \alpha \AA$. The
  proof of the other equivalence is similar.

  First, we prove that $\alpha (\AA)^\omega \aimply \alpha \AA$.
  Let $\Tau$ be a tree and $\varphi$ be an LTL formula such that $\Tau
  \models \alpha (\AA)^\omega . \varphi$.
  Moreover, let $\xi$ be any strategy such that all $p \in
  \Out_\Tau(\alpha(\AA)^\omega, \xi)$ satisfy $\varphi$.
  In order to prove that $\Tau \models \alpha \AA . \varphi$ it is
  sufficient to use the strategy $\xi$ in the moves of player E, namely,
  whenever we need to prove that $p \models \EE \alpha' . \varphi$
  according to Definition~\ref{d:ae-ltl-sem}, we take $p' = \xi(p)$
  and we move to prove that $p' \models \alpha' . \varphi$.
  In this way, the infinite paths selected by
  Definition~\ref{d:ae-ltl-sem} for $\alpha \AA$ coincide with the
  possible outcomes of game $\alpha (\AA)^\omega$, and hence satisfy
  the LTL formula $\varphi$.
  
  This concludes the proof that $\alpha (\AA)^\omega \aimply \alpha
  \AA$.
  We now prove that $\alpha \AA \aimply \alpha (\AA)^\omega$.
  We distinguish three cases.
  \begin{itemize}
   \item \textbf{Case $\alpha = (\AA)^n$, with $n \geq 0$}.
  
    In this case, $\alpha \AA \aequiv \AA$
    (Lemma~\ref{l:ae-ltl-prop}(1)) and $\alpha (\AA)^\omega =
    (\AA)^\omega$.
    Let $\Tau$ be a tree and $\varphi$ be an LTL formula.
    Then $\Tau \models \AA . \varphi$ if and only if all the paths
    of $\Tau$ satisfy formula $\varphi$.
    It is easy to check that also $\Tau \models (\AA)^\omega .
    \varphi$ if and only if all the paths of $\Tau$ satisfy formula
    $\varphi$.
    This is sufficient to conclude that $(\AA)^n \AA \aequiv (\AA)^n
    (\AA)^\omega$.

   \item \textbf{Case $\alpha = \EE \alpha'$}.
  
    In this case, $\alpha \AA \aequiv \EE\AA$. Indeed, $\alpha \AA$ is
    an arbitrary path quantifier that starts with $\EE$ and ends with
    $\AA$.
    By Lemma~\ref{l:ae-ltl-prop}(1), we can collapse adjacent
    occurrences of $\AA$ and of $\EE$ , thus obtaining $\alpha \AA
    \aequiv (\EE\AA)^n$ for some $n > 0$.
    Moreover, by Lemma~\ref{l:ae-ltl-prop}(4) we have $(\EE\AA)^n
    \aequiv \EE\AA$.
    
    Let $\Tau$ be a tree and $\varphi$ be an LTL formula.
    Then $\Tau \models \EE \AA . \varphi$ if and only if there is
    some finite path $\bar p$ of $\Tau$ such that all the infinite
    extensions of $\bar p$ satisfy $\varphi$.
    Now, let $\xi$ be any strategy such that $\xi(\epsilon) = \bar
    p$.
    Then every infinite path $p \in \Out_\Tau(\EE\alpha'(\AA)^\omega,
    \xi)$ satisfies $\varphi$.
    Indeed, since player E has the first turn, all the possible
    outcomes are infinite extensions of $\xi(\epsilon) = \bar
    p$.
    
    This concludes the proof that $\EE \alpha' \AA \aimply \EE \alpha'
    (\AA)^\omega$.
  
   \item \textbf{Case $\alpha = (\AA)^n\EE\alpha'$, with $n > 0$}.
  
    Reasoning as in the proof of the previous case, it is easy to show
    that $\alpha \AA \aequiv \AA\EE\AA$.
    
    Let $\Tau$ be a tree and $\varphi$ be an LTL formula.
    Then $\Tau \models \AA \EE \AA . \varphi$ if and only if for
    every finite path $p$ of $\Tau$ there is some finite extension
    $p'$ of $p$ such that all the infinite extensions of $p'$
    satisfy the formula $\varphi$.
    Let $\xi$ be any strategy such that $p' = \xi(p)$ is a finite
    extension of $p$ such that all the infinite extensions of $p'$
    satisfy $\varphi$.
    Then every infinite path $p \in
    \Out_\Tau((\AA)^n\EE\alpha'(\AA)^\omega, \xi)$
    satisfies $\varphi$.
    Indeed, let $p_0,p_1,\ldots,p_n,p_{n+1},\ldots$ be a generating
    sequence for $p$. Then $p_{n+1} = \xi(p_n)$ and $p$ is an infinite
    extension of $p_{n+1}$. By construction of $\xi$ we know that $p$
    satisfies $\varphi$.
    
    This concludes the proof that $(\AA)^n\EE\alpha' \AA \aimply
    (\AA)^n\EE\alpha' (\AA)^\omega$.
  \end{itemize}

  Every finite path quantifier $\alpha$ falls in one of
  the three considered cases.
  Therefore, we can conclude that $\alpha \AA \aimply \alpha
  (\AA)^\omega$ for every finite path quantifier $\alpha$.
  \qed
\end{proof}

The next lemma defines a sufficient condition for proving that $\alpha
\aimply \alpha'$. This condition is useful for the proofs of the
forthcoming lemmas.
\begin{lemma}
  \label{l:ae-ltl-inf-imply}
  Let $\alpha$ and $\alpha'$ be two infinite path quantifiers.
  Let us assume that for all $\Sigma$-labelled trees and for each strategy
  $\xi$ there is some strategy $\xi'$ such that
  $\Out_\Tau(\alpha',\xi') \subseteq \Out_\Tau(\alpha,\xi)$.
  Then $\alpha \aimply \alpha'$.
\end{lemma}
\begin{proof}
  Let us assume that $\Tau \models \alpha . \varphi$. 
  Then there is a suitable strategy $\xi$ such that all $p \in
  \Out_\Tau(\alpha,\xi)$ satisfy the LTL formula $\varphi$.
  Let $\xi'$ be a strategy such that all $\Out_\Tau(\alpha',\xi')
  \subseteq \Out_\Tau(\alpha,\xi)$.
  By hypothesis, all possible outcomes for game $\alpha'$ and
  strategy $\xi'$ satisfy the LTL formula $\varphi$, and hence $\Tau
  \models \alpha' . \varphi$.
  This concludes the proof that $\alpha \aimply \alpha'$.
  \qed
\end{proof}

In the next lemma we show that all the games where players A and E
alternate infinitely often are equivalent to one of the two games
$(\AA\EE)^\omega$ and $(\EE\AA)^\omega$.
That is, we can assume that each player extends the path only once
before the turn passes to the other player.
\begin{lemma}
  \label{l:alternating}
  Let $\alpha$ be an infinite path quantifier that contains an
  infinite number of $\AA$ and an infinite number of $\EE$. Then
  $\alpha \aequiv ({\AA\EE})^\omega$ or $\alpha \aequiv
  ({\EE\AA})^\omega$.
\end{lemma}
\begin{proof}
  Let $\alpha = (\AA)^{m_1} (\EE)^{n_1} (\AA)^{m_2} (\EE)^{n_2}
  \cdots$ with $m_i,n_i > 0$.
  We show that $\alpha \aequiv ({\AA\EE})^\omega$.
  
  First, we prove that $({\AA\EE})^\omega \aimply \alpha$.
  Let $\xi$ be a strategy for the tree $\Tau$ and let $p$ be an infinite
  path of $\Tau$. We show that if $p \in \Out_\Tau(\alpha,\xi)$ then
  $p \in \Out_\tau((\AA\EE)^\omega,\xi)$. By
  Lemma~\ref{l:ae-ltl-inf-imply} this is sufficient for proving that  
  $({\AA\EE})^\omega \aimply \alpha$.
  
  Let $p_0,p_1,\ldots$ be a generating sequence for $p$ according to
  $\alpha$ and $\xi$.
  Moreover, let $p'_0 = \epsilon$, $p'_{2i+1} =
  p_{m_1+n_1+\cdots+m_{i-1}+n_{i-1}+m_i}$ and  and $p'_{2i+2}
  = p_{m_1+n_1+\cdots+m_{i-1}+n_{i-1}+m_i+1}$.
  It is easy to check that $p'_0, p'_1, p'_2, \ldots$ is a valid
  generating sequence for $p$ according to game $(\AA \EE)^\omega$ and
  strategy $\xi$.
  Indeed, extensions $p'_0 \to p'_1$, $p'_2 \to p'_3$, $p'_4 \to
  p'_5$, \dots{} are moves of player A, and hence can be arbitrary.
  Extensions $p'_1 \to p'_2$, $p'_3 \to p'_4$, \dots{} correspond to
  extensions $p_{m_1} \to p_{m_1+1}$, $p_{m_1+n_1+m_2} \to
  p_{m_1+n_1+m_2+1}$, \dots, which are moves of player E and hence
  respect strategy $\xi$.
  
  We now prove that $\alpha \aimply ({\AA\EE})^\omega$.
  Let $\xi$ be a strategy for the tree $\Tau$. We define a strategy
  $\bar \xi$ such that if $p \in \Out_\Tau((\AA\EE)^\omega, \bar
  \xi)$, then $p\in \Out_\Tau(\alpha,\xi)$.
  By Lemma~\ref{l:ae-ltl-inf-imply} this is sufficient for proving
  that $\alpha \aimply ({\AA\EE})^\omega$.
  
  Let $\bar p$ be a finite path. Then $\bar \xi(\bar p) = \xi^{k_{\bar
  p}}(\bar p)$ with $k_{\bar p} = \sum_{i=1}^{|\bar p|} n_i$. That is,
  strategy $\bar \xi$ on path $\bar p$ is obtained by applying
  $k_{\bar p}$ times strategy $\xi$. The number of times strategy
  $\xi$ is applied depends on the length $|\bar p|$ of path $\bar p$.
  
  We show that, if $p$ is a possible outcome of the game $\alpha$ with
  strategy $\bar \xi$, then $p$ is a possible outcome of the game
  $(\AA\EE)^\omega$ with strategy $\xi$.
  Let $p_0, p_1, \ldots$ be a generating sequence for $p$ according to
  $({\AA\EE})^\omega$ and $\bar \xi$.
  Then 
  \begin{gather*}
    p_0,
    \underbrace{p_1,..., p_1}_{\text{$m_1$ times}}, 
    \underbrace{\xi(p_1), \xi^2(p_1),..., \xi^{n_1}(p_1)}_{\text{$n_1$ times}}, 
    \underbrace{p_3,..., p_3}_{\text{$m_2$ times}},\hspace{1cm}\\
    \hspace{2cm} \underbrace{\xi(p_3), \xi^2(p_3),..., \xi^{n_2}(p_3)}_{\text{$n_2$ times}},
    \underbrace{p_5,..., p_5}_{\text{$m_3$ times}},
    ...
  \end{gather*}
  is a valid generating sequence for $p$ according to $\alpha$ and $\xi$.
  The extensions corresponding to an occurrence of symbol $\EE$ in
  $\alpha$ consist of an application of the strategy $\xi$ 
  and are hence valid for player E.
  Moreover, extension $\xi^{n_i}(p_{2i-1}) \to p_{2i+1}$ is a valid
  move for player A because $p_{2i+1}$ is an extension of
  $\xi^{n_i}(p_{2i-1})$.
  Indeed, $\xi^{n_i}(p_{2i-1})$ is a prefix of $p_{2i}$ (and hence of
  $p_{2i+1}$) since $p_{2i} = \bar \xi(p_{2i-1}) =
  \xi^{k_{p_{2i-1}}}(p_{2i-1})$ and $k_{p_{2i-1}} =
  \sum_{x=1}^{|p_{2i-1}|} n_x \geq n_i$, since $|p_{2i-1}| \geq i$.
  The other conditions of Definition~\ref{d:ae-ltl-inf-sem} can be
  easily checked.
  
  This concludes the proof that $\alpha \aequiv ({\AA\EE})^\omega$ for
  $\alpha = (\AA)^{m_1} (\EE)^{n_1} (\AA)^{m_2} (\EE)^{n_2} \cdots$.
  The proof that $\alpha \aequiv ({\EE\AA})^\omega$ for $\alpha =
  (\EE)^{m_1} (\AA)^{n_1} (\EE)^{m_2} (\AA)^{n_2} \cdots$ is similar.
  \qed
\end{proof}

The next lemma contains other auxiliary results on path quantifiers.
\begin{lemma}
  \label{l:ae-ltl-prop-bis}
  Let $\alpha$ be a finite path quantifier and $\alpha'$ be an
  infinite path quantifier.
  \begin{enumerate}
   \item $\alpha \AA \alpha' \aimply \alpha \alpha'$ and $\alpha
    \alpha' \aimply \alpha \EE \alpha'$.
   \item $\alpha (\AA)^\omega \aimply \alpha \AA \alpha'$ and $\alpha
    \EE \alpha' \aimply \alpha (\EE)^\omega$.
  \end{enumerate}
\end{lemma}
\begin{proof}
  \begin{enumerate}
   \item We prove that $\alpha \AA \alpha' \aimply \alpha \alpha'$.
    Let $\xi$ be a strategy for tree $\Tau$ and let $p$ be an infinite
    path of $\Tau$. We show that if $p \in \Out_\Tau(\alpha \alpha',
    \xi)$ then $p \in \Out_\Tau(\alpha \AA \alpha', \xi)$.
    Let $p_0,p_1,\ldots$ be a generating sequence for $p$ according to
    $\alpha \alpha'$ and $\xi$.
    Then it is easy to check that $p_0, p_1, \ldots, p_{i-1}, p_i,
    p_i, p_{i+1}, \ldots$, where $i$ is the length of $\alpha$, is a
    valid generating sequence for $p$ according to $\alpha \AA
    \alpha'$ and $\xi$.
    Indeed, the extension $p_i \to p_i$ is a valid move for player A.
    This concludes the proof that $\alpha \AA \alpha' \aimply \alpha \alpha'$.

    Now we prove that $\alpha \alpha' \aimply \alpha \EE \alpha'$.
    If $\alpha' = (\EE)^\omega$, then $\alpha \EE \alpha' = \alpha \EE
    (\EE)^\omega = \alpha (\EE)^\omega = \alpha \alpha'$, and 
    $\alpha \EE \alpha' \aimply \alpha \alpha'$ is trivially true.
    If $\alpha' \neq (\EE)^\omega$, we can assume, without loss of
    generality, that $\alpha' = \AA \alpha''$.
    In this case, let $\xi$ be a strategy for tree $\Tau$ and let $p$ be a
    path of $\Tau$. We show that if $p \in \Out_\Tau( \alpha \EE
    \alpha', \xi)$ then $p \in \Out_\Tau(\alpha \alpha',\xi)$.
    Let $p_0,p_1,\ldots$ be a generating sequence for $p$ according to
    $\alpha \EE \alpha'$ and $\xi$.
    Then it is easy to check that $p_0, p_1, \ldots, p_i, p_{i+2},
    \ldots$, where $i$ is the length of $\alpha$, is a valid
    generating sequence for $p$ according to $\alpha \alpha'$ and $\xi$.
    Indeed, extension $p_i \to p_{i+2}$ is valid, as it
    corresponds to the first symbol of $\alpha'$ and we have assumed
    it to be symbol $\AA$.
    This concludes the proof that $\alpha \alpha' \aimply \alpha \EE \alpha'$.
    
   \item We prove that $\alpha (\AA)^\omega \aimply \alpha \alpha'$.
    The proof that $\alpha \alpha' \aimply \alpha (\EE)^\omega$ is
    similar.

    Let $\xi$ be a strategy for tree $\Tau$ and let $p$ be an infinite
    path of $\Tau$. We show that if $p \in \Out_\Tau(\alpha
    (\AA)^\omega, \xi) $ then $p \in \Out_\Tau(\alpha \alpha', \xi)$.
    Let $p_0,p_1,\ldots$ be a generating sequence for $p$ according to
    $\alpha \alpha'$ and $\xi$.
    Then it is easy to check that $p_0, p_1, \ldots$ is a valid
    generating sequence for $p$ according to $\alpha (\AA)^\omega$ and
    $\xi$.
    In fact, $\alpha (\AA)^\omega$ defines less restrictive conditions
    on generating sequences than $\alpha \alpha'$.

    This is sufficient to conclude that $\alpha (\AA)^\omega \aimply
    \alpha \alpha'$.
    \qed
  \end{enumerate}
\end{proof}

We can now complete the picture of Theorem~\ref{t:123}: each finite or
infinite path quantifier is equivalent to a \emph{canonical path
quantifier} that defines a game consisting of alternated moves of
players A and E of length one, two, three, or infinity.
\begin{theorem}
  \label{t:123inf}
  For each finite or infinite path quantifier $\alpha$ there is a
  canonical path quantifier $$\alpha' \in \{ \AA, \EE, \AA\EE, \EE\AA,
  \AA\EE\AA, \EE\AA\EE, ({\AA\EE})^\omega, ({\EE\AA})^\omega \}$$
  such that $\alpha \aequiv \alpha'$.
  Moreover, the following implications hold between the
  canonical path quantifiers:
  \begin{equation}
    \label{eq:123inf}
    \begin{xy}
      \xymatrix{
        \AA \ar@{~>}[r] &
        \AA\EE\AA \ar@{~>}[r]\ar@{~>}[d] &
        (\AA\EE)^\omega \ar@{~>}[r]\ar@{~>}[d] &
        \AA\EE \ar@{~>}[d] \\
        &
        \EE\AA \ar@{~>}[r] &
        (\EE\AA)^\omega \ar@{~>}[r] &
        \EE\AA\EE \ar@{~>}[r] &
        \EE
      }
    \end{xy}
  \end{equation}
\end{theorem}
\begin{proof}
  We first prove that each path quantifier is equivalent to a
  canonical path quantifier.
  By Theorem~\ref{t:123}, this is true for the finite path
  quantifiers, so we only consider infinite path quantifiers.
  
  Let $\alpha$ be an infinite path quantifier. We distinguish three
  cases:
  \begin{itemize}
   \item $\alpha$ contains an infinite number of $\AA$ and an infinite
    number of $\EE$: then, by Lemma~\ref{l:alternating}, $\alpha$
    is equivalent to one of the canonical games $({\AA\EE})^\omega$ or
    $({\EE\AA})^\omega$.
   \item $\alpha$ contains a finite number of $\AA$: in this case,
    $\alpha$ ends with an infinite sequence of $\EE$, and, by
    Lemma~\ref{l:ae-ltl-inf-fin}, $\alpha \aequiv \alpha''$ for
    some finite path quantifier $\alpha''$. By Theorem~\ref{t:123},
    $\alpha''$ is equivalent to some canonical path quantifier, and
    this concludes the proof for this case.
   \item $\alpha$ contains a finite number of $\EE$: this case is
    similar to the previous one.
  \end{itemize}

  For the implications in (\ref{eq:123inf}):
  \begin{itemize}
   \item $({\AA\EE})^\omega \aimply ({\EE\AA})^\omega$ comes from
    Lemma~\ref{l:ae-ltl-prop-bis}(1), by taking the empty word for
    $\alpha$ and $\alpha' = (\AA\EE)^\omega$.
   \item $\AA\EE\AA \aimply ({\AA\EE})^\omega$, $({\AA\EE})^\omega
    \aimply \AA\EE$, $\EE\AA \aimply ({\EE\AA})^\omega$, and
    $({\EE\AA})^\omega \aimply \EE\AA\EE$ come from
    Lemmas~\ref{l:ae-ltl-inf-fin} and~\ref{l:ae-ltl-prop-bis}(2).
   \item The other implications come from Theorem~\ref{t:123}.
    \qed
  \end{itemize}
\end{proof}


\subsection{Strictness of the Implications}

We conclude this section by showing that all the arrows in the diagram
of Theorem~\ref{t:123inf} describe strict implications, namely,
the eight canonical path quantifiers are all different.
Let us consider the following $\{i,p,q\}$-labelled binary tree, where
the root is labelled by $i$ and each node has two children labelled
with $p$ and $q$:
\begin{equation*}
  \newcommand{\state}[1]{*+[o][F-]{#1}}
  \xymatrix @C=0.1pc @R=1.2pc {
  &&&&&&& \state{i} \ar[rrrrd]\ar[lllld] \\
  &&& \state{p} \ar[rrd]\ar[lld] &&&&&&&&
  \state{q} \ar[rrd]\ar[lld] \\
  & \state{p} \ar[rd]\ar[ld] &&&& \state{q} \ar[rd]\ar[ld] &&&&
  \state{p} \ar[rd]\ar[ld] &&&& \state{q} \ar[rd]\ar[ld] \\
  \state{p} \ar@{.}[d] && \state{q} \ar@{.}[d] &&
  \state{p} \ar@{.}[d] && \state{q} \ar@{.}[d] &&
  \state{p} \ar@{.}[d] && \state{q} \ar@{.}[d] &&
  \state{p} \ar@{.}[d] && \state{q} \ar@{.}[d] \\
  &&&&&&&&&&&&&&&
  }
\end{equation*}
Let us consider the following LTL formulas:
\begin{itemize}
 \item $\F p$: player E can satisfy this formula if he moves at least
  once, by visiting a $p$-labelled node.
 \item $\G\F p$: player E can satisfy this formula if he can visit
  an infinite number of $p$-labelled nodes, that is, if he
  has the final move in a finite game, or if he moves infinitely often
  in an infinite game.
 \item $\F\G p$: player E can satisfy this formula only if he takes
  control of the game from a certain point on, that is, only if he has
  the final move in a finite game.
 \item $\G\neg q$: player E can satisfy this formula only if player A
  never plays, since player A can immediately visit a
  $q$-labelled node.
 \item $\X p$: player E can satisfy this formula by playing the
  first turn and moving to the left child of the root node.
\end{itemize}
The following graph shows which formulas hold for which path
quantifiers:
\begin{equation*}
  \begin{xy}
    \xymatrix @C=1.2pc {
      & \F p & \G\F p & \F\G p & \G \neg q \\
      \AA \ar@{~>}[r] &
      \AA\EE\AA \ar@{~>}[r]\ar@{~>}[d] &
      (\AA\EE)^\omega \ar@{~>}[r]\ar@{~>}[d] &
      \AA\EE \ar@{~>}[d] \\
      \X p &
      \EE\AA \ar@{~>}[r] &
      (\EE\AA)^\omega \ar@{~>}[r] &
      \EE\AA\EE \ar@{~>}[r] &
      \EE 
      \save
      "3,5"!/d15pt/."1,5"*\frm<8pt>{.}
           ."1,4"*\frm<8pt>{.}
           ."1,3"*\frm<8pt>{.}
           ."1,2"*\frm<8pt>{.},
      "3,5"."3,1"*+\frm<8pt>{.}
      \restore
    }
  \end{xy}
\end{equation*}


\section{A Planning Algorithm for $\AE$-LTL}
\label{sec:ae-ltl-planning}

In this section we present a planning algorithm for $\AE$-LTL goals.
We start by showing how to build a parity tree automaton that accepts
all the trees that satisfy a given $\AE$-LTL formula.
Then we show how this tree automaton can be adapted, so that it
accepts only trees that 
correspond to valid plans for a given planning domain.
In this way, the problem of checking whether there exists some plan
for a given domain and for an $\AE$-LTL goal is reduced to the emptiness
problem on tree automata.
Finally, we study the complexity of planning for $\AE$-LTL goals and
we prove that this problem is 2EXPTIME-complete.


\subsection{Tree Automata and $\AE$-LTL Formulas}
\label{ssec:ae-ltl-automata}

\citeA{BGK03} have shown that $\AE$-LTL formulas can be expressed
directly as \CTLstar formulas. The reduction exploits the equivalence
of expressive power of \CTLstar and monadic path logic \cite{MR99}.
A tree automaton can be obtained for an $\AE$-LTL formula using this reduction
and Theorem~\ref{t:ctlstar-tree-aut}.
However, the translation proposed by \citeA{BGK03} has an upper bound of
non-elementary complexity, and is hence not useful for our complexity
analysis.
In this paper we describe a different, more direct reduction that is
better suited for our purposes.

A $\Sigma$-labelled tree $\Tau$ satisfies a formula $\alpha . \varphi$
if there is a suitable subset of paths of the tree that satisfy
$\varphi$. The subset of paths should be chosen according to $\alpha$.
In order to characterize the suitable subsets of paths, we assume to
have a $w$-marking of the tree $\Tau$, and we use the labels $w$ to
define the selected paths.

\begin{definition}[$w$-marking]
  A $w$-marking of the $\Sigma$-labelled tree $\Tau$ is a
  $(\Sigma{\times}\{w,\overline w\})$-la\-belled tree $\Tauw$ such that
  $\dom(\Tau) = \dom(\Tauw)$ and, whenever $\Tau(x) = \sigma$, then
  $\Tauw(x) = (\sigma, w)$ or $\Tauw(x) = (\sigma, \overline w)$.
\end{definition}
We exploit $w$-markings as follows. We associate to each $\AE$-LTL
formula $\alpha .  \varphi$ a \CTLstar formula $\pf{\alpha.\varphi}$
such that the tree $\Tau$ satisfies the formula $\alpha . \varphi$ if
and only if there is a $w$-marking of $\Tau$ that satisfies
$\pf{\alpha .  \varphi}$.

\begin{definition}[$\AE$-LTL and \CTLstar]
  \label{d:ae-ltl-vs-ctl}
  Let $\alpha . \varphi$ be an $\AE$-LTL formula. The \CTLstar formula
  $\pf{\alpha.\varphi}$ is defined as follows:
  \begin{eqnarray*}
    \pf{\AA . \varphi} & = & \Ap \varphi                    \\
    \pf{\EE . \varphi} & = & \Ep \varphi                    \\
    \pf{\EE\AA . \varphi}
    & = & \EF w  \,\wedge\, \Ap(\F w \to \varphi)           \\
    \pf{\AA\EE\AA . \varphi}
    & = & \AG \EF w \,\wedge\, \Ap(\F w \to \varphi)        \\
    \pf{\AA\EE . \varphi}
    & = & \AG \EXG w \,\wedge\, \Ap(\F\G w \to \varphi)     \\
    \pf{\EE\AA\EE . \varphi}
    & = & \EF \AG \EXG w \,\wedge\, \Ap(\F\G w \to \varphi) \\
    \pf{(\AA\EE)^\omega . \varphi}
    & = & \AG \EF w \,\wedge\, \Ap(\G\F w \to \varphi)      \\
    \pf{(\EE\AA)^\omega . \varphi}
    & = & \EF \AG \EF w \,\wedge\, \Ap(\G\F w \to \varphi)
  \end{eqnarray*}
\end{definition}
In the case of path quantifiers $\AA$ and $\EE$, there is a direct
translation into \CTLstar that does not exploit the $w$-marking.
In the other cases, the \CTLstar formula $\pf{\alpha.\varphi}$ is the
conjunction of two sub-formulas. The first one characterizes the good
markings according to the path quantifier $\alpha$, while the second
one guarantees that the paths selected according to the marking
satisfy the LTL formula $\varphi$.
In the case of path quantifiers $\EE\AA$ and $\AA\EE\AA$, we
mark with $w$ the nodes that, once reached, guarantee that
the formula $\varphi$ is satisfied. The selected paths are hence those
that contain a node labelled by $w$ (formula $\F w$).
In the case of path quantifiers $\AA\EE$ and $\EE\AA\EE$, we
mark with $w$ all the descendants of a node that define an
infinite path that satisfies $\varphi$.  The selected paths are hence
those that, from a certain node on, are continuously labelled by $w$
(formula $\F\G w$). 
In the case of path quantifiers $(\AA\EE)^\omega$ and
$(\EE\AA)^\omega$, finally, we mark with $w$ all the nodes
that player E wants to reach according to its strategy before passing
the turn to player A.  The selected paths are hence those that contain
an infinite number of nodes labelled by $w$ (formula $\G\F w$), that
is, the paths along which player E moves infinitely often.

\begin{theorem}
  \label{t:abstract-tree}
  A $\Sigma$-labelled tree $\Tau$ satisfies the $\AE$-LTL formula
  $\alpha.\varphi$ if and only if there is some $w$-marking of
  $\Tau$ that satisfies formula $\pf{\alpha.\varphi}$.
\end{theorem}
\begin{proof}
  In the proof, we consider only the cases of $\alpha = \AA\EE\AA$,
  $\alpha = \AA\EE$ and $\alpha = (\AA\EE)^\omega$. The other cases
  are similar.
  
  Assume that a tree $\Tau$ satisfies $\alpha . \varphi$. Then we show
  that there exists a $w$-marking $\Tauw$ of $\Tau$ that
  satisfies $\pf{\alpha . \varphi}$.

  \begin{itemize}
   \item \textbf{Case $\alpha = \AA\EE\AA$.}
    According to Definition~\ref{d:ae-ltl-sem}, if the tree $\Tau$
    satisfies $\AA\EE\AA . \varphi$, then every finite path $p$ of
    $\Tau$ can be extended to a finite path $p'$ such that all the
    infinite extensions $p''$ of $p'$ satisfy $\varphi$. Let us mark
    with $w$ all the nodes of $\Tauw$ that correspond to the extension
    $p'$ of some path $p$.
    By construction, the marked tree satisfies $\AG\EF w$.  It remains
    to show that the marked tree satisfies $\Ap (\F w \to \varphi)$.
 
    Let us consider any path $p''$ in the tree that satisfies $\F w$,
    and let us show that $p''$ also satisfies $\varphi$.
    Since $p''$ satisfies $\F w$, we know that it contains nodes
    marked with $w$.
    Let $p'$ be the finite prefix of path $p''$ up to the first node
    marked by $w$. By construction, there exists a finite path $p$
    such that $p'$ is a finite extension of $p$ and all the infinite
    extensions of $p'$ satisfy $\varphi$. As a consequence, also $p''$
    satisfies $\varphi$.
    
   \item \textbf{Case $\alpha = \AA\EE$.}
    According to Definition~\ref{d:ae-ltl-sem}, if the tree $\Tau$
    satisfies $\AA\EE . \varphi$, then for all the finite paths $p$
    there is some infinite extension of $p$ that satisfies
    $\varphi$.
    Therefore, we can define a mapping $m: P^*(\Tau) \to
    P^\omega(\Tau)$ that associates to a finite path $p$ an infinite
    extension $m(p)$ that satisfies $\varphi$.
    We can assume, without loss of generality, that, if $p'$ is a
    finite extension of $p$ and is also a prefix of $m(p)$, then
    $m(p') = m(p)$. That is, as far as $p'$ extends the finite path
    $p$ along the infinite path $m(p)$ then $m$ associates to $p'$ the
    same infinite path $m(p)$.
    
    For every finite path $p$, let us mark with $w$ the node of
    $\Tauw$ that is the child of $p$ along the infinite path $m(p)$.
    By construction, the marked tree satisfies $\AG\EXG w$.
    It remains to show that the marked tree satisfies $\Ap (\F\G w \to
    \varphi)$.
 
    Let us consider a path $p''$ in the tree that satisfies $\F\G w$,
    and let us show that $p''$ also satisfies $\varphi$.
    Since $p''$ satisfies $\F\G w$, we know that there is some path
    $p$ such that all the descendants of $p$ along $p''$ are marked
    with $w$.
    In order to prove that $p''$ satisfies $\varphi$ we show that $p''
    = m(p)$. Assume by contradiction that $m(p) \neq p''$ and let $p'$
    be the longest common prefix of $m(p)$ and $p''$. We observe that
    $p$ is a prefix of $p'$, and hence $m(p) = m(p')$.  This implies
    that the child node of $p'$ along $p''$ is not marked with $w$,
    which is absurd, since by definition of $p$ all the descendants of
    $p$ along $p''$ are marked with $w$.
    
   \item \textbf{Case $\alpha = (\AA\EE)^\omega$.}
    According to Definition~\ref{d:ae-ltl-inf-sem}, if the tree $\Tau$
    satisfies $(\AA\EE)^\omega . \varphi$, then there exists a
    suitable strategy $\xi$ for player E so that all the possible
    outcomes of game $\alpha$ with strategy $\xi$ satisfy $\varphi$.
    Let us mark with $w$ all the nodes in $\Tauw$ that correspond to
    the extension $\xi(p)$ of some finite path $p$.
    That is, we mark with $w$ all the nodes that are reached after
    some move of player E according to strategy $\xi$.
    The marked tree satisfies the formula $\AG \EF w$, that is, every
    finite path $p$ can be extended to a finite path $p'$ such that
    the node corresponding to $p'$ is marked with $w$. Indeed, by
    construction, it is sufficient to take $p' = \xi(p'')$ for some
    extension $p''$ of $p$.
    It remains to show that the marked tree satisfies $\Ap (\G\F w \to
    \varphi)$.
 
    Let us consider a path $p$ in the tree that satisfies $\G\F w$, and
    let us show that $p$ also satisfies $\varphi$.
    To this purpose, we show that $p$ is a possible outcome of game
    $\alpha$ with strategy $\xi$.
    We remark that, given an arbitrary finite prefix $p'$ of $p$ it is
    always possible to find some finite extension $p''$ of $p'$
    such that $\xi(p'')$ is also a prefix of $p$. Indeed, the set of
    paths $P = \{ \bar p \st \text{$\xi(\bar p)$ is a finite prefix of
    $p$}\}$ is infinite, as there are infinite nodes marked with $w$
    in path $p$. 
  
    Now, let $p_0, p_1, p_2, \ldots$ be the sequence of finite paths
    defined as follows: $p_0 = (\epsilon)$ is the root of the three;
    $p_{2k+1}$ is the shortest extension of $p_{2k}$ such that
    $\xi(p_{2k+1})$ is a prefix of $p$; and $p_{2k+2} = \xi(p_{2k+1})$.
    It is easy to check that $p_0, p_1, p_2, \ldots$ is a generating
    sequence for $p$ according to $(\AA\EE)^\omega$ and $\xi$.
    Hence, by Definition~\ref{d:ae-ltl-inf-sem}, the infinite
    path $p$ satisfies the LTL formula $\varphi$.
  \end{itemize}
    
  This concludes the proof that if $\Tau$ satisfies $\alpha .
  \varphi$, then there exists a $w$-marking of $\Tau$ that
  satisfies $\pf{\alpha . \varphi}$.
  
  Assume now that there is a $w$-marked tree $\Tauw$
  that satisfies $\pf{\alpha . \varphi}$. We show that $\Tau$
  satisfies $\alpha . \varphi$.

  \begin{itemize}
   \item \textbf{Case $\alpha = \AA\EE\AA$.}
    The marked tree satisfies the formula $\AG\EF w$. This means that
    for each finite path $p$ ($\AG$) there exists some finite
    extension $p'$ such that the final node of $p'$ is marked
    by $w$ ($\EF w$) .
    Let $p''$ be any infinite extension of such a finite path $p'$. We
    show that $p''$ satisfies the LTL formula $\varphi$. Clearly,
    $p''$ satisfies the formula $\F w$.
    Since the tree satisfies the formula $\Ap(\F w \to \varphi)$, all the
    infinite paths that satisfy $\F w$ also satisfy 
    $\varphi$. Therefore, $p''$ satisfies the LTL formula $\varphi$.
  
   \item \textbf{Case $\alpha = \AA\EE$.}
    The marked tree satisfies the formula $\AG\EXG w$. Then,
    for each finite path $p$ ($\AG$) there exists some infinite
    extension $p'$ such that, from a certain node on, all the
    nodes of $p'$ are marked with $w$ ($\EXG w$).
    We show that, if $p'$ is the infinite extension of some finite
    path $p$, then $p'$ satisfies the LTL formula $\varphi$. Clearly,
    $p'$ satisfies the formula $\F\G w$.
    Since the tree satisfies the formula $\Ap(\F\G w \to \varphi)$, all the
    infinite paths that satisfy $\F\G w$ also satisfy
    $\varphi$. Therefore, $p'$ satisfies the LTL formula $\varphi$.
  
   \item \textbf{Case $\alpha = (\AA\EE)^\omega$.}
    Let $\xi$ be any strategy so that, for every finite path $p$, the
    node corresponding to $\xi(p)$ is marked with $w$.
    We remark that it is always possible to define such a strategy. In
    fact, the marked tree satisfies the formula $\AG\EF w$, and hence,
    each finite path $p$ can be extended to a finite path $p'$ such
    that the node corresponding to $p'$ is marked with $w$.
  
    Let $p$ be a possible outcome of game $\alpha$ with strategy
    $\xi$.  We should prove that $p$ satisfies the LTL formula
    $\varphi$.
    By Definition~\ref{d:ae-ltl-inf-sem}, the infinite path $p$
    contains an infinite set of nodes marked by $w$: these are all the
    nodes reached after a move of player E.
    Hence, $p$ satisfies the formula $\G\F w$.
    Since the tree satisfies the formula $\Ap(\G\F w \to \varphi)$, all the
    infinite paths that satisfy $\G\F w$ also satisfy 
    $\varphi$. Therefore, path $p$ satisfies the LTL formula
    $\varphi$.
  \end{itemize}
  
  This concludes the proof that, if there exists a $w$-marking
  of tree $\Tau$ that satisfies $\pf{\alpha.\varphi}$,
  then $\Tau \models \alpha .  \varphi$.
  \qed
\end{proof}

\citeA{Kup99} defines an extension of \CTLstar with existential
quantification over atomic propositions (EG\CTLstar) and examines
complexity of model checking and satisfiability for the new logic.
We remark that $\AE$-LTL can be seen as a subset of EG\CTLstar. Indeed,
according to Theorem~\ref{t:abstract-tree}, a $\Sigma$-labelled tree
satisfies an $\AE$-LTL formula $\alpha. \varphi$ if and only if it
satisfies the EG\CTLstar formula $\exists w . \pf{\alpha.\varphi}$.

In the following definition we show how to transform a parity tree
automaton for the \CTLstar formula $\pf{\alpha.\varphi}$ into a parity
tree automaton for the $\AE$-LTL formula $\alpha . \varphi$.
This transformation is performed by abstracting away the information
on the $w$-marking from the input alphabet and from
the transition relation of the tree automaton.

\begin{definition}
  \label{d:abstract-mark}
  Let $A = \langle \Sigma {\times} \{w, \overline w\}, \Dcal, Q, q_0,
  \delta, \beta \rangle$ be a parity tree automaton.
  The parity tree automaton $A_{\exists w} = \langle \Sigma, \Dcal, Q,
  q_0, \delta_{\exists w}, \beta \rangle$, obtained from $A$ by
  abstracting away the $w$-marking, is defined as
  follows: $\delta_{\exists w}(q,\sigma,d) = \delta(q,(\sigma,w),d)
  \cup \delta(q,(\sigma,\overline w),d)$.
\end{definition}
\begin{lemma}
  \label{l:abstract-aut}
  Let $A$ and $A_{\exists w}$ be two parity tree automata as
  in Definition~\ref{d:abstract-mark}.
  $A_{\exists w}$ accepts exactly the $\Sigma$-labelled trees that
  have some $w$-marking which is accepted by $A$.
\end{lemma}
\begin{proof}
  Let $\Tauw$ be a $(\Sigma{\times}\{w,\overline w\})$-labelled tree
  and let $\Tau$ be the corresponding $\Sigma$-labelled tree,
  obtained by abstracting away the $w$-marking.
  We show that if $\Tauw$ is accepted by $A$, then $\Tau$ is accepted
  by $A_{\exists w}$.
  Let $r: \tau \to Q$ be an accepting run of $\Tauw$ on $A$. Then $r$
  is also an accepting run of $\Tau$ on $A_{\exists w}$.
  Indeed, if $x \in \tau$, $\arity(x) = d$, and $\Tauw(x) = (\sigma,m)$
  with $m \in \{w,\overline w\}$, then we have $\langle r(x\cdot0), \ldots,
  r(x\cdot{d{-}1}) \rangle \in \delta(r(x), (\sigma,m),d)$.
  Then $\Tau(x) = \sigma$, and, by definition of $A_{\exists w}$, we
  have $\langle r(x\cdot0), \ldots, r(x\cdot{d{-}1}) \rangle \in
  \delta_{\exists w}(r(x),\sigma,d)$.
  
  Now we show that, if the $\Sigma$-labelled tree $\Tau$ is accepted
  by $A_{\exists w}$, then there is a $(\Sigma{\times}\{w,\overline
  w\})$-labelled tree $\Tauw$ that is a $w$-marking of
  $\Tau$ and that is accepted by $A$.
  Let $r: \tau \to Q$ be an accepting run of $\Tau$ on $A_{\exists
  w}$.
  By definition of run, we know that if $x \in \tau$, with $\arity(x)
  = d$ and $\Tau(x) = \sigma$, then $\langle r(x\cdot0), \ldots,
  r(x\cdot{d{-}1}) \rangle \in \delta_{\exists w}(r(x),\sigma,d)$.
  By definition of $\delta_{\exists w}$, we know that $\langle
  r(x\cdot0), \ldots, r(x\cdot{d{-}1}) \rangle \in \delta
  (r(x),(\sigma,w),d) \cup \delta (r(x),(\sigma,\overline w),d)$.
  Let us define $\Tauw(x) = (\sigma,w)$ if $\langle r(x\cdot0),
  \ldots, r(x\cdot{d{-}1}) \rangle \in \delta (r(x),(\sigma,w),d)$,
  and $\Tauw(x) = (\sigma,\overline w)$ otherwise.
  It is easy to check that $r$ is an accepting run of $\Tauw$ on
  $A$.
  \qed
\end{proof}

Now we have all the ingredients for defining the tree automaton that
accepts all the trees that satisfy a given $\AE$-LTL formula.

\begin{definition}[tree automaton for $\AE$-LTL]
  \label{d:algo}
  Let $\Dcal \subseteq \Nat^*$ be a finite set of arities, and let
  $\alpha .  \varphi$ be an $\AE$-LTL formula.
  The parity tree automaton $A^\Dcal_{\alpha . \varphi}$ is obtained
  by applying the transformation described in
  Definition~\ref{d:abstract-mark} to the parity automaton
  $A^\Dcal_{\pf{\alpha . \varphi}}$ built according to
  Theorem~\ref{t:ctlstar-tree-aut}.
\end{definition}

\begin{theorem}
  \label{t:ae-ltl-tree-aut}
  The parity tree automaton $A^\Dcal_{\alpha . \varphi}$ accepts 
  exactly the $\Sigma$-labelled $\Dcal$-trees that satisfy the formula
  $\alpha .  \varphi$.
\end{theorem}
\begin{proof}
  By Theorem~\ref{t:ctlstar-tree-aut}, the parity tree automaton
  $A^\Dcal_{\pf{\alpha . \varphi}}$ accepts all the $\Dcal$-trees that
  satisfy the \CTLstar formula $\pf{\alpha . \varphi}$.
  Therefore, the parity tree automaton $A^\Dcal_{\alpha.\varphi}$
  accepts all the $\Dcal$-trees that satisfy the formula $\alpha .
  \varphi$ by Lemma~\ref{l:abstract-aut} and
  Theorem~\ref{t:abstract-tree}.
  \qed
\end{proof}

The parity tree automaton $A^\Dcal_{\alpha . \varphi}$ has a parity
index that is exponential and a number of states that is doubly
exponential in the length of formula $\varphi$.
\begin{proposition}
  \label{p:ae-ltl-tree-aut-size}
  The parity tree automaton $A^\Dcal_{\alpha . \varphi}$ has
  $2^{2^{O(|\varphi|)}}$ states and parity index $2^{O(|\varphi|)}$.
\end{proposition}
\begin{proof}
  The construction of Definition~\ref{d:abstract-mark} does not change
  the number of states and the parity index of the automaton.
  Therefore, the proposition follows from
  Theorem~\ref{t:ctlstar-tree-aut}.
  \qed
\end{proof}


\subsection{The Planning Algorithm}
\label{ssec:ae-ltl-algo}

We now describe how the automaton $A^\Dcal_{\alpha . \varphi}$ can be
exploited in order to build a plan for goal $\alpha . \varphi$ on a
given domain.

We start by defining a tree automaton that accepts all the trees that
define the valid plans of a planning domain $D = \langle
\Sigma,\sigma_0,A,R \rangle$. We recall that, according to
Definition~\ref{d:domain}, transition relation $R$ maps a state
$\sigma \in \Sigma$ and an action $a \in A$ into a tuple of next
states $\langle \sigma_1, \sigma_2, \ldots, \sigma_n \rangle =
R(\sigma,a)$.
 
In the following we assume that $\Dcal$ is a finite set of arities
that is compatible with domain $D$, namely, if $R(\sigma,a) = \langle
\sigma_1, \ldots, \sigma_d \rangle$ for some $\sigma \in \Sigma$ and
$a \in A$, then $d \in \Dcal$.

\begin{definition}[tree automaton for a planning domain]
  \label{d:aut-dom}
  Let $D = \langle \Sigma, \sigma_0, A, R \rangle$ be a planning
  domain and let $\Dcal$ be a set of arities that is compatible with
  domain $D$.
  The tree automaton $A^\Dcal_D$ corresponding to the planning domain
  is $A^\Dcal_D = \langle \Sigma{\times}A, \Dcal, \Sigma, \sigma_0,
  \delta_D, \beta_0 \rangle$, where $\langle \sigma_1, \ldots,
  \sigma_d \rangle \in \delta_D(\sigma,(\sigma,a),d)$ if $\langle
  \sigma_1, \ldots, \sigma_d \rangle = R(\sigma,a)$ with $d > 0$, and
  $\beta_0(\sigma) = 0$ for all $\sigma \in \Sigma$.
\end{definition}

According to Definition~\ref{d:execution-tree}, a
$(\Sigma{\times}A)$-labelled tree can be obtained from each plan $\pi$
for domain $D$.
Now we show that also the converse is true, namely, each
$(\Sigma{\times}A)$-labelled tree accepted by the tree automaton
$A^\Dcal_D$ induces a plan.
\begin{definition}[plan induced by a tree]
  \label{d:plan-tree} \mbox{}
  Let $\Tau$ be a $(\Sigma{\times}A)$-labelled tree that is accepted
  by automaton $A^\Dcal_D$.
  The \emph{plan $\pi$ induced by $\Tau$} on
  domain $D$ is defined as follows:
  $\pi(\sigma_0, \sigma_1, \ldots, \sigma_n) = a$ if there is some
  finite path $p$ in $\Tau$ with $\Tau(p) = (\sigma_0,a_0) \cdot
  (\sigma_1,a_1) \cdots (\sigma_n,a_n)$ and $a = a_n$.
\end{definition}

The following lemma shows that Definitions~\ref{d:execution-tree}
and~\ref{d:plan-tree} define a one-to-one correspondence between the
valid plans for a planning domain $D$ and the trees accepted by
automaton $A^\Dcal_D$.

\begin{lemma}
  \label{l:exec-induced-plan}
  Let $\Tau$ be a tree accepted by automaton $A^\Dcal_D$ and let $\pi$
  be the corresponding induced plan.
  Then $\pi$ is a valid plan for domain $D$, and $\Tau$ is the
  execution tree corresponding to $\pi$.
  Conversely, let $\pi$ be a plan for domain $D$ and let $\Tau$ be the
  corresponding execution structure.
  Then $\Tau$ is accepted by automaton $A^\Dcal_D$ and $\pi$ is the
  plan induced by $\Tau$.
\end{lemma}
\begin{proof}
  This lemma is a direct consequence of
  Definitions~\ref{d:execution-tree} and~\ref{d:plan-tree}.
  \qed
\end{proof}

We now define a parity tree automaton that accepts only the trees that
correspond to the plans for domain $D$ and that satisfy goal $g = \alpha .
\varphi$.
This parity tree automaton is obtained by combining in a suitable way
the tree automaton for $\AE$-LTL formula $g$ (Definition~\ref{d:algo})
and the tree automaton for domain $D$ (Definition~\ref{d:aut-dom}).

\begin{definition}[instrumented tree automaton]
  \label{d:instrumented-tree-aut}
  Let $\Dcal$ be a set of arities that is compatible with planning
  domain $D$.
  Let also $A^\Dcal_g = \langle \Sigma, \Dcal, Q, q_0, \delta, \beta
  \rangle$ be a parity tree automaton that accepts only the trees that
  satisfy the $\AE$-LTL formula $g$.
  The parity tree automaton $A^\Dcal_{D,g}$ corresponding to planning domain
  $D$ and goal $g$ is defined as follows:
  $A^\Dcal_{D,g} = \langle \Sigma{\times}A, \Dcal, Q{\times}\Sigma,
  (q_0,\sigma_0), \delta', \beta' \rangle$, where $\langle
  (q_1,\sigma_1), \ldots, (q_d,\sigma_d) \rangle \in
  \delta'((q,\sigma),(\sigma,a),d)$ if $\langle q_1, \ldots, q_d \rangle
  \in \delta(q,\sigma,d)$ and
  $\langle \sigma_1, \ldots, \sigma_d \rangle = R(\sigma,a)$ with $d >
  0$, and where $\beta'(q,\sigma) = \beta(q)$.
\end{definition}

The following lemmas show that solutions to planning problem
$(D,g)$ are in one-to-one correspondence with the trees accepted by
the tree automaton $A^\Dcal_{D,g}$.

\begin{lemma}
  \label{l:tree-to-plan}
  Let $\Tau$ be a $(\Sigma{\times}A)$-labelled tree that is accepted
  by automaton $A^\Dcal_{D,g}$, and let $\pi$ be the plan induced by
  $\Tau$ on domain $D$.
  Then the plan $\pi$ is a solution to planning problem $(D,g)$.
\end{lemma}
\begin{proof}
  According to Definition~\ref{d:planning-problem}, we have to prove
  that the execution tree corresponding to $\pi$ satisfies the goal $g$.
  By Lemma~\ref{l:exec-induced-plan}, this amounts to proving that the tree
  $\Tau$ satisfies $g$.
  By construction, it is easy to check that if a
  $(\Sigma{\times}A)$-labeled tree $\Tau$ is accepted by $A^\Dcal_{D,g}$, then
  it is also accepted by $A^{\Dcal}_g$.
  Indeed, if $r_{D,g}: \tau \to {Q \times \Sigma}$ is an accepting run
  of $\Tau$ on $A^\Dcal_{D,g}$, then $r_g : \tau \to Q$ is an accepting run
  of $\Tau$ on $A^{\Dcal}_g$, where $r_g(x) = q$ whenever $r_{D,g} =
  (q, \sigma)$ for some $\sigma \in \Sigma$.
  \qed
\end{proof}

\begin{lemma}
  \label{l:plan-to-tree}
  Let $\pi$ be a solution to planning problem $(D,g)$.
  Then the execution tree of $\pi$ is accepted by automaton
  $A^\Dcal_{D,g}$.
\end{lemma}
\begin{proof}
  Let $\Tau$ be the execution tree of $\pi$.
  By Lemma~\ref{l:exec-induced-plan} we know that $\Tau$ is accepted
  by $A^\Dcal_D$.
  Moreover, by definition of solution of a planning problem, we know
  that $\Tau$ is accepted also by $A^\Dcal_g$.
  By construction, it is easy to check that if a
  $(\Sigma{\times}A)$-labeled tree $\Tau$ is accepted by $A^\Dcal_{D}$ and by
  $A^{\Dcal}_g$, then it is also accepted by $A^\Dcal_{D,g}$.
  Indeed, let $r_D: \tau \to \Sigma$ be an accepting run of $\Tau$ on
  $A^\Dcal_D$ and let $r_g: \tau \to Q$ be an accepting run of $\Tau$ on
  $A^{\Dcal}_g$. Then $r_{D,g}: \tau \to {Q \times \Sigma}$ is an
  accepting run of $\Tau$ on $A^\Dcal_{D,g}$, where $r_{D,g}(x) =
  (q,\sigma)$ if $r_D(x) = \sigma$ and $r_g(x) = q$.
  \qed
\end{proof}

As a consequence, checking whether
goal $g$ can be satisfied on domain $D$ is reduced to the problem of
checking whether automaton $A^\Dcal_{D,g}$ is nonempty.
\begin{theorem}
  \label{t:planning-tree-aut}
  Let $D$ be a planning domain and $g$ be an $\AE$-LTL formula.
  A plan exists for goal $g$ on domain $D$ if and only if the tree
  automaton $A^\Dcal_{D,g}$ is nonempty.
\end{theorem}

\begin{proposition}
  \label{p:planning-tree-aut-size}
  The parity tree automaton $A^\Dcal_{D,g}$ for domain $D = (\Sigma,
  \sigma_0, A, R)$ and goal $g = \alpha.\varphi$ has $|\Sigma| \cdot
  2^{2^{O(|\varphi|)}}$ states and parity index $2^{O(|\varphi|)}$.
\end{proposition}
\begin{proof}
  This is a consequence of Proposition~\ref{p:ae-ltl-tree-aut-size}
  and of the definition of automaton $A^\Dcal_{D,g}$.
  \qed
\end{proof}


\subsection{Complexity}
\label{ssec:ae-ltl-complexity}

We now study the time complexity of the planning algorithm defined in
Subsection~\ref{ssec:ae-ltl-algo}.

Given a planning domain $D$, the planning problem for
$\AE$-LTL goals $g = \alpha . \varphi$ can be decided in a time that
is doubly exponential in the size of the formula $\varphi$ by applying
Theorem~\ref{t:tree-aut-complexity} to the tree automaton $A^\Dcal_{D,g}$.

\begin{lemma}
  \label{l:ae-ltl-planning-algo-complexity}
  Let $D$ be a planning domain.
  The existence of a plan for $\AE$-LTL goal $g = \alpha . \varphi$ on
  domain $D$ can be decided in time $2^{2^{O(|\varphi|)}}$.
\end{lemma}
\begin{proof}
  By Theorem~\ref{t:planning-tree-aut} the existence of a plan for
  goal $g$ on domain $D$ is reduced to the emptiness problem on parity
  tree automaton $A^\Dcal_{D,g}$.
  By Proposition~\ref{p:planning-tree-aut-size}, the parity tree
  automaton $A^\Dcal_{D,g}$ has $2^{2^{O(|\varphi|)}} \times |\Sigma|$
  states and parity index $2^{O(|\varphi|)}$.
  Since we assume that domain $D$ is fixed,  
  by Theorem~\ref{t:tree-aut-complexity}, the emptiness of automaton
  $A^\Dcal_{D,g}$ can be decided in time $2^{2^{O(|\varphi|)}}$.
  \qed
\end{proof}

The doubly exponential time bound is tight.
Indeed, the \emph{realizability problem} for an LTL formula $\varphi$,
which is known to be 2EXPTIME-complete \cite{PR90}, can be reduced to a
planning problem for the goal $\AA . \varphi$.
In a realizability problem one assumes that a program and the
environment alternate in the control of the evolution of the system.
More precisely, in an execution $\sigma_0, \sigma_1, \ldots$ the
states $\sigma_i$ are decided by the program if $i$ is even, and by
the environment if $i$ is odd.  We say that a given formula $\varphi$
is realizable if there is some program such that all its executions
satisfy $\varphi$ independently on the actions of the environment.

\begin{theorem}
  \label{t:ae-ltl-planning-complexity}
  Let $D$ be a planning domain.
  The problem of deciding the existence of a plan for $\AE$-LTL goal
  $g = \alpha . \varphi$ on domain $D$ is 2EXPTIME-complete.
\end{theorem}
\begin{proof}
  The realizability of formula $\varphi$ can be reduced to the problem
  of checking the existence of a plan for goal $\AA . \varphi$ on
  planning domain $D = \big( \{\init\} \cup (\Sigma \times \{p,e\}),
  \init, \Sigma \cup \{e\}, R \big)$, with:
  \begin{align*}
    R(\init,\sigma') & = \{(\sigma',e)\} &
    R(\init,e) & = \emptyset \\
    R((\sigma,p),\sigma') & = \{(\sigma',e)\} &
    R((\sigma,p),e) & = \emptyset \\
    R((\sigma,e),\sigma') & = \emptyset &
    R((\sigma,e),e) & = \{ (\sigma',p) \st \sigma' \in \Sigma\}
  \end{align*}
  for all $\sigma,\sigma' \in \Sigma$.

  States $(\sigma,p)$ are those where the program controls the
  evolution through actions $\sigma' \in \Sigma$.
  States $(\sigma,e)$ are those where the environment controls the
  evolution; only the nondeterministic action $e$ can be performed in
  this state.
  Finally, state $\init$ is used to assign the initial move to the
  program.
  
  Since the realizability problem is 2EXPTIME-complete in the size of
  the LTL formula \cite{PR90}, the planning problem is 2EXPTIME-hard
  in the size of the goal $g = \alpha .  \varphi$.
  The 2EXPTIME-completeness follows from
  Lemma~\ref{l:ae-ltl-planning-algo-complexity}.
  \qed
\end{proof}

We remark that, in the case of goals of the form $\EE . \varphi$, an
algorithm with a better complexity can be defined.
In this case, a plan exists for $\EE . \varphi$ if and only if
there is an infinite sequence $\sigma_0, \sigma_1, \ldots$ of states
that satisfies $\varphi$ and such that $\sigma_{i+1} \in
R(\sigma_i,a_i)$ for some action $a_i$.
That is, the planning problem can be reduced to
a model checking problem for LTL formula $\varphi$,
and this problem is known to be PSPACE-complete \cite{SC85}.
We conjecture that, for all the canonical path quantifiers $\alpha$
except $\EE$, the doubly exponential bound of
Theorem~\ref{t:ae-ltl-planning-complexity} is tight.

Some remarks are in order on the complexity of the
\emph{satisfiability} and \emph{validity problems} for $\AE$-LTL
goals. These problems are PSPACE-complete.
Indeed, the $\AE$-LTL formula $\alpha . \varphi$ is satisfiable if
and only if the LTL formula $\varphi$ is satisfiable\footnote{If a
tree satisfies $\alpha .  \varphi$ then some of its paths satisfy
$\varphi$, and a path that satisfies $\varphi$ can be seen also as a
tree that satisfies $\alpha .  \varphi$.}, and the latter problem is
known to be PSPACE-complete \cite{SC85}.
A similar argument holds also for validity.

The complexity of the \emph{model checking problem} for $\AE$-LTL has
been recently addressed by \citeA{KV06}.
\citeauthor{KV06} introduce m\CTLstar, a variant of \CTLstar, where path
quantifiers have a ``memoryful'' interpretation. They show that memoryful
quantification can express (with linear cost) the semantics of path
quantifiers in our $\AE$-LTL. For example, the $\AE$-LTL formula
$\AA\EE.\varphi$ is expressed in m\CTLstar by the formula $\AG\Ep\varphi$.
\citeauthor{KV06} show that the model checking problem for the new logic is
EXPSPACE-complete, and that this result holds also for the subset of
m\CTLstar that corresponds to formulas $\AA\EE . \varphi$.
Therefore, the model checking problem for $\AE$-LTL with finite path
quantifiers is also EXPSPACE-complete. To the best of our knowledge
the complexity of model checking $\AE$-LTL formulas $(\AA\EE)^\omega
. \varphi$ and $(\EE\AA)^\omega . \varphi$ is still an open problem.


\section{Two Specific Cases: Reachability and Maintainability Goals}
\label{sec:reach-and-maintain}

In this section we consider two basic classes of goals that are
particularly relevant in the field of planning.

\subsection{Reachability Goals}

The first class of goals are the \emph{reachability goals}
corresponding to the LTL formula $\F q$, where $q$ is a propositional
formula.
Most of the literature in planning concentrates on this class of
goals, and there are several works that address the problem of
defining plans of different strength for this kind of goals (see,
e.g., \citeR{CPRT03} and their citations).

In the context of $\AE$-LTL, as soon as player E takes control, it can
immediately achieve the reachability goal if possible at all. The fact
that the control is given back to player A after the goal has been
achieved is irrelevant.  Therefore, the only significant path quantifiers
for reachability goals are $\AA$, $\EE$, and $\AA\EE$.

\begin{proposition}
  \label{l:reach}
  Let $q$ be a propositional formula on atomic propositions $\mathit{Prop}$.
  Then, the following results hold for every labelled tree $\Tau$.
  $\Tau \models \EE . \F q$ iff $\Tau \models \EE\AA . \F q$ iff
  $\Tau \models \EE\AA\EE . \F q$ iff $\Tau \models (\EE\AA)^\omega . \F
  q$.
  Moreover $\Tau \models \AA\EE . \F q$ iff $\Tau \models \AA\EE\AA . \F
  q$ iff $\Tau \models (\AA\EE)^\omega . \F q$.
\end{proposition}
\begin{proof}
  We prove that $\Tau \models \AA\EE . \F q$ iff $\Tau \models
  \AA\EE\AA . \F q$ iff $\Tau \models (\AA\EE)^\omega . \F q$.
  The other cases are similar.
  
  Let us assume that $\Tau \models \AA\EE . \F q$.
  Moreover, let $p$ be a finite path of $\Tau$.
  We know that $p$ can be extended to an infinite path $p'$ such that
  $\Tau(p') \models \F q$.
  According to the semantics of LTL, $\Tau(p') \models \F q$ means
  that there is some node $x$ in path $p'$ such that $q \in \Tau(x)$.
  Clearly, all infinite paths of $\Tau$ that contain node $x$ also
  satisfy the LTL formula $\F q$.
  Therefore, there is a finite extension $p''$ of $p$ such that all the
  infinite extensions of $p''$ satisfy the LTL formula $\F q$: it is
  sufficient to take as $p''$ an finite extension of $p$ that
  contains node $x$.
  Since this property holds for every finite path $p$, we conclude
  that $\Tau \models \AA\EE\AA . \F q$.
  
  We have proven that $\Tau \models \AA\EE . \F q$ implies $\Tau
  \models \AA\EE\AA . \F q$.
  By Theorem~\ref{t:123inf} we know that $\AA\EE\AA \aimply
  (\AA\EE)^\omega \aimply \AA\EE$, and hence $\Tau \models \AA\EE\AA .
  \F q$ implies $\Tau \models (\AA\EE)^\omega . \F q$ implies $\Tau
  \models \AA\EE . \F q$.
  This concludes the proof.
  \qed
\end{proof}

The following diagram shows the implications among the significant path
quantifiers for reachability goals:
\begin{equation}
  \begin{xy}
    \xymatrix{
    \AA \ar@{~>}[r] & \AA\EE \ar@{~>}[r] & \EE
    }
  \end{xy}
\end{equation}
We remark that the three goals $\AA . \F q$, $\EE . \F q$, and $\AA \EE
. \F q$ correspond, respectively, to the strong, weak, and strong cyclic
planning problems of \citeA{CPRT03}.


\subsection{Maintainability Goals}

We now consider another particular case, namely the maintainability
goals $\G q$, where $q$ is a propositional formula.
Maintainability goals have properties that are complementary to the
properties of reachability goals.
In this case, as soon as player A takes control, it can violate the
maintainability goal if possible at all. The fact that player E can
take control after player A is hence irrelevant, and  the only
interesting path quantifiers are $\AA$, $\EE$, and $\EE\AA$.

\begin{proposition}
  \label{l:maintain}
  Let $q$ be a propositional formula on atomic propositions $\mathit{Prop}$.
  Then, the following results hold for every labelled tree $\Tau$.
  Then $\Tau \models \AA . \G q$ iff $\Tau \models \AA\EE . \G q$ iff
  $\Tau \models \AA\EE\AA . \G q$ iff $\Tau \models (\AA\EE)^\omega . \G
  q$.
  Moreover $\Tau \models \EE\AA . \G q$ iff $\Tau \models \EE\AA\EE . \G
  q$ iff $\Tau \models (\EE\AA)^\omega . \G q$.
\end{proposition}
\begin{proof}
  The proof is similar to the proof of Proposition~\ref{l:reach}.
  \qed
\end{proof}

The following diagram shows the implications among the significant path
quantifiers for maintainability goals:
\begin{equation*}
  \begin{xy}
    \xymatrix{
    \AA \ar@{~>}[r] & \EE\AA \ar@{~>}[r] & \EE
    }
  \end{xy}
\end{equation*}
The goals $\AA . \G q$, $\EE . \G q$, and $\EE\AA . \G q$ correspond to
maintainability variants of strong, weak, and strong cyclic planning
problems. Indeed, they correspond to requiring that condition $q$ is
maintained for all evolutions despite nondeterminism ($\AA . \G q$),
that condition $q$ is maintained for some of the evolutions ($\EE . \G
q$), and that it is possible to reach a state where condition $q$ is
always maintained despite nondeterminism ($\EE\AA . \G p$).


\section{Related Works and Concluding Remarks}
\label{sec:conl}

In this paper we have defined $\AE$-LTL, a new temporal logic that
extends LTL with the possibility of declaring complex path quantifiers
that define the different degrees in which an LTL formula can be
satisfied by a computation tree.
We propose to use $\AE$-LTL formulas for expressing temporally
extended goals in nondeterministic planning domains.
We have defined a planning algorithm for $\AE$-LTL goals that is based
on an automata-theoretic framework: the existence of a plan is reduced
to checking the emptiness of a suitable parity tree automaton.
We have studied the time complexity of the planning algorithm, proving
that it is 2EXPTIME-complete in the length of the $\AE$-LTL formula.

In the field of planning, several works use temporal logics for
defining goals.
Most of these approaches \cite{BK98,BK00,CdGLV02,CM98,dGV99,KD01} use
linear temporal logics as the goal language, and are not able to
express conditions on the degree in which the goal should be satisfied
with respect to the nondeterminism in the execution.
Notable exceptions are the works described by \citeA{PBT01,PT01} and by
\citeA{DLPT02}. \citeA{PBT01} and \citeA{PT01} use CTL as goal language,
while \citeA{DLPT02} define a new branching time logic that
allows for expressing temporally extended goals that can deal
explicitly with failure and recovery in goal achievement.
In these goal languages, however, path
quantifiers are interleaved with the temporal operators, and are hence
rather different from $\AE$-LTL.

In the field of temporal logics, the work on alternating temporal
logic (ATL) \cite{AHK02} is related to our work. In ATL, the path
quantifiers in CTL and \CTLstar are replaced by game quantifiers.
Nevertheless, there is no obvious way to expressed formulas of the
form $\alpha.\varphi$, where $\alpha$ is a path quantifier and $\varphi$ is
an LTL formula in ATL$^*$, which is the most expressive logic studied
by \citeA{AHK02}. Our conjecture is that our logic and ATL$^*$ are
of incomparable expressiveness.

Some comments are in order on the practical impact of the 2EXPTIME
complexity of the planning algorithm.
First of all, in many planning problems we expect to have very complex
and large domains, but goals that are relatively simple (see, e.g.,
the experimental evaluation performed by \citeA{PBT01} in the case of
planning goals expressed as CTL formulas).
In these cases, the doubly exponential complexity of the algorithm in
the size of the formula may not be a bottleneck.
For larger $\AE$-LTL goals, a doubly exponential time complexity may
not be feasible, but it should be noted that this is
\emph{worst-case} complexity.  We also note that improved algorithms
for plan synthesis is an active research area, including the analysis of
simpler LTL goals \cite{AT04} and the development of improved
automata-theoretic algorithms \cite{KV05}.

The automata-theoretic framework that we have used in the paper is of
wider applicability than $\AE$-LTL goals. An interesting direction for
future investigations is the application of the framework to variants
of $\AE$-LTL that allow for nesting of path quantifiers, or for
goals that combine $\AE$-LTL with propositional or temporal operators.
This would allow, for instance, to specify goals which compose
requirements of different strength. A simple example of such goals is
$(\AA\EE . \F p) \wedge (\AA . \G p)$, which requires to achieve
condition $p$ in a strong cyclic way, maintaining condition $q$ in a
strong way. The impossibility to define such kind of goals is, in our
opinion, the strongest limitation of $\AE$-LTL with respect to 
CTL and \CTLstar.

Another direction for future investigations is the extension of the
approach proposed in this paper to the case of planning under partial
observability \cite{dGV99}, where one assumes that the agent executing
the plan can observe only part of the state and hence its choices on
the actions to execute may depend only on that part.

We also plan to explore implementation issues and, in particular, the
possibility of exploiting BDD-based symbolic techniques in a planning
algorithm for $\AE$-LTL goals. In some cases, these techniques have
shown to be able to deal effectively with domains and goals of a
significant complexity, despite the exponential worst-case time
complexity of the problems \cite{BCPRT01,PBT01}.


\acks{A shorter version of this paper, without proofs, has been
published by \citeA{PV03}.
The authors would like to thank Erich Gr\"adel for his comments on the
reduction of $\AE$-LTL formulas to \CTLstar formulas.}


\bibliographystyle{theapa}
\bibliography{ltlplan}


\end{document}